\documentclass[10pt]{article} 
\usepackage[accepted]{tmlr}

\usepackage{amsmath,amsfonts,bm}

\def\eqref#1{equation~\ref{#1}}

\def\1{\bm{1}}

\def\rc{{\textnormal{c}}}

\def\rf{{\textnormal{f}}}

\DeclareMathAlphabet{\mathsfit}{\encodingdefault}{\sfdefault}{m}{sl}
\SetMathAlphabet{\mathsfit}{bold}{\encodingdefault}{\sfdefault}{bx}{n}

\usepackage{hyperref}
\usepackage{url}
\usepackage{booktabs}       
\usepackage{amsfonts}       
\usepackage{nicefrac}       
\usepackage{microtype}      
\usepackage{xcolor} 
\usepackage{graphicx}
\usepackage{caption}
\usepackage{subcaption}
\usepackage{style}
\usepackage{algorithm} 
\usepackage{algorithmicx}
\usepackage{algpseudocode}
\usepackage{wrapfig} 
\usepackage{mathtools}
\usepackage{array}
\usepackage{colortbl}
\usepackage{unicode-math}
\usepackage{enumitem}
\setlist{leftmargin=5.5mm}
\usepackage{pifont}
\renewcommand{\eqref}[1]{(\ref{#1})}
\newcommand*\colourcheck[1]{%
  \expandafter\newcommand\csname #1check\endcsname{\textcolor{#1}{\ding{52}}}%
}
\newcommand*\colourxcheck[1]{%
  \expandafter\newcommand\csname #1xcheck\endcsname{\textcolor{#1}{\ding{55}}}%
}
\newcommand{\revs}[1]{{\color{black} #1}}
\colourcheck{blue}
\colourcheck{green}
\colourcheck{red}
\colourxcheck{blue}
\colourxcheck{green}
\colourxcheck{red}

\title{On the Convergence Rates of Federated Q-Learning across Heterogeneous Environments}
\author{%
\name Leo (Muxing) Wang 
\email wang.muxin@northeastern.edu\\
\addr Northeastern University
\AND 
Pengkun Yang  
\email yangpengkun@tsinghua.edu.cn\\
\addr Tsinghua University 
\AND
Lili Su 
\email l.su@northeastern.edu\\
\addr Northeastern University
}

\def\month{08}  
\def\year{2025} 
\def\openreview{\url{https://openreview.net/forum?id=EkLAG3gt3g}} 

\begin{document}

\maketitle

\begin{abstract}
Large-scale multi-agent systems are often deployed across wide geographic areas, where agents interact with heterogeneous environments. 
There is an emerging interest in understanding the role of heterogeneity in the performance of the federated versions of classic reinforcement learning algorithms. In this paper, we study synchronous federated Q-learning, which aims to learn an optimal Q-function by having $K$ agents average their local Q-estimates per $E$ iterations.   
We provide a characterization of the error evolution, which decays to zero as the number of iterations $T$ increases.  
We show that when $K(E-1)$ is below a certain threshold, similar to the homogeneous environment settings, there is a linear speed-up concerning $K$.  
In sharp contrast, when $K(E-1)$ is above the threshold, heterogeneous environments lead to significant performance degradation. In particular, as $E$ increases, the convergence rate deteriorates. 
The slow convergence of having $E>1$ turns out to be fundamental rather than an artifact of our analysis. 
We prove that, 
for a wide range of stepsizes, the $\ell_{\infty}$ norm of the error cannot decay faster than { $\Theta_R \pth{E/{((1-\gamma)T)}}$, where $\Theta_R$ only hides numerical constants and the specific choice of reward values}.   
In addition, our experiments demonstrate that the convergence exhibits an interesting two-phase phenomenon. For any given stepsize, there is a sharp phase transition of the convergence: the error decays rapidly in the beginning yet later bounces up and stabilizes. 
\end{abstract}
\section{Introduction}
\label{sec:intro}
Advancements in unmanned capabilities are rapidly transforming industries and national security by enabling fast-paced and versatile operations across domains such as advanced manufacturing \citep{park2019reinforcement}, autonomous driving \citep{kiran2021deep}, and battlefields \citep{mohlenhof2021reinforcement}. 
Reinforcement learning (RL) -- a cornerstone for unmanned capabilities -- is a powerful machine learning method that aims to enable an agent to learn an optimal policy via interacting with its operating environment to solve sequential decision-making problems \citep{bertsekas1996neuro,bertsekas2019reinforcement}. 
However, the ever-increasing complexity of the environment results in a high-dimensional state-action space, often imposing overwhelmingly high sample collection requirements on individual agents.  
This limited-data challenge becomes a significant hurdle that must be addressed to realize the potential of reinforcement learning.

In this paper, we study reinforcement learning within a federated learning framework (also known as Federated Reinforcement Learning \citep{Qi_2021,jin_federated_2022,woo_blessing_2023}), wherein multiple agents independently collect samples and collaboratively train a common policy under the orchestration of a parameter server without disclosing the local data trajectories.   A simple illustration can be found in Fig.\,\ref{fig: system setup}. 
When the environments of all agents are homogeneous, it has been shown that the federated version of classic reinforcement learning algorithms 
can significantly alleviate the data collection burden on individual agents \citep{woo_blessing_2023,khodadadian2022federated} --  error bounds derived therein exhibit a linear speedup in terms of the number of agents. 
Moreover, by tuning the synchronization period 
$E$ (i.e., the number of iterations between agent synchronization), the communication cost can be significantly reduced compared with $E=1$ 
\begin{wrapfigure}[13]{r}{0.5\textwidth}  
\centering
\resizebox{\linewidth}{!}{
\includegraphics[width=0.9\linewidth]{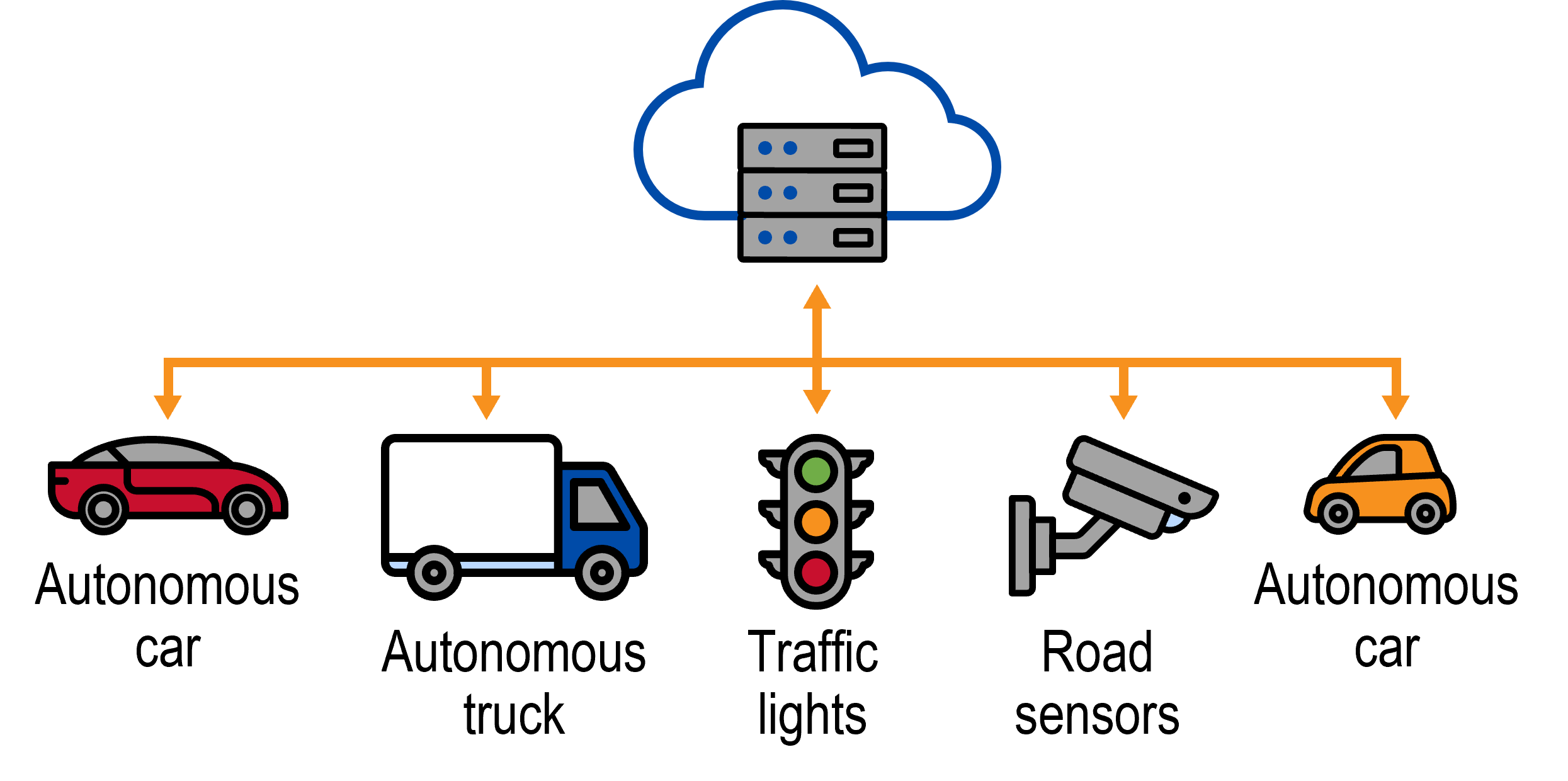} }
\caption{An illustration of a federated learning system. }
\label{fig: system setup}
\end{wrapfigure} 
yet without significant performance degradation.   
However, many large-scale multi-agent systems are often deployed across wide geographic areas, resulting in agents interacting with heterogeneous environments. For instance, connected and autonomous vehicles (CAVs) operating in various regions of a metropolitan area encounter diverse conditions such as varying traffic patterns, road infrastructure, and local regulations. 
Intuitively, the environmental heterogeneity may lead to misaligned learning signals across agents, potentially hinder the convergence, and degrade the generalization performance of the learned policies.  Hence, the clients' federation must be managed in a way that ensures the learned policy is robust to environmental heterogeneity.

There is an emerging interest in mathematically understanding the role of heterogeneity in the performance of the federated versions of classic reinforcement learning algorithms \citep{jin_federated_2022,woo_blessing_2023,doan2019finite,wang2023federated,xie_fedkl_2023} such as Q-learning, policy gradient methods, and temporal difference (TD) methods.  
In this paper, we study synchronous federated Q-learning (FQL) in the presence of environmental heterogeneity, which aims to learn an optimal Q-function by averaging local Q-estimates per $E$ (where $E\ge 1$) update iterations on their local data. 
We leave the exploration of asynchronous Q-learning for future work. 
Federated Q-learning is a natural integration of FedAvg and Q-learning \citep{jin_federated_2022,woo_blessing_2023}. The former is the most widely adopted classic federated learning algorithm \citep{kairouz2021advances,mcmahan2017communication}, and the latter is one of the most fundamental model-free reinforcement learning algorithms \citep{Watkins1992Qlearning}. 
Despite intensive study, the tight sample complexity of Q-learning in the single-agent setting was open until recently \citep{li2024q}.  
Similarly, the understanding of FedAvg is far from complete; a detailed discussion can be found in Section \ref{sec: related_work}. A concise comparison of our work to the related work can be found in Table \ref{tbl:comparison}. 

{\bf Contributions.}  
Our contributions can be summarized as follows.  
All the asymptotic notations, e.g., $\calO$ and $\tilde\calO$, unless otherwise specified, hide only numerical constants. 

\vskip -0.8\baselineskip 
\vspace{-0em}
\begin{itemize}\setlength\itemsep{0em} 
\item We characterize the error evolution of synchronous Federated Q-learning, showing that it decays to zero as the number of iterations $T$ increases.  
{When $K(E-1)$ is below a threshold of $\tilde{\calO}\pth{{(\kappa\epsilon)^{-1}(1-\gamma)^{-2}}}$,}
similar to the homogeneous environment settings, there is a linear speed-up concerning $K$ and the sample complexity is $\tilde\calO\pth{\frac{|\calS||\calA|}{K(1-\gamma)^5\epsilon^2}}$, matching the homogenous setting \citep{woo_blessing_2023}.  
Here, $\mathcal{S}$ and $\mathcal{A}$ are the state and action spaces, $\gamma\in (0,1)$ denotes the discount factor, {and $\kappa$ is a scalar characterizing the environment heterogeneity}. 
In sharp contrast, when $K(E-1)$ is above the threshold, heterogeneous environments lead to significant performance degradation {and results in a unique sample complexity of  $\tilde{\mathcal{O}}\left(\frac{|\mathcal{S}||\mathcal{A}|\kappa E}{(1-\gamma)^3\epsilon}\right)$}. {Note that $E$ is also $\epsilon$-dependent. Hence, the sample complexity $\tilde{\mathcal{O}}\left(\frac{|\mathcal{S}||\mathcal{A}|\kappa E}{(1-\gamma)^3\epsilon}\right)$ does not contradict the lower bound of $\Omega(\epsilon^{-2})$ for the single-agent case.}
\item  We prove that the convergence slowing down for $E>1$ is fundamental. We show that the $\ell_{\infty}$ norm of the error cannot decay faster than {$\Theta_R \pth{\frac{E}{(1-\gamma)T}}$, where $\Theta_R$ only hides numerical constants and the specific choice of reward values}.  
A practical implication of this impossibility result is that, eventually, having multiple local updates (i.e., $E>1$) ends up consuming more samples (i.e., $E\times$ more) than using $E=1$ to reach a target accuracy. 
\item Our numerical results illustrate that when the environments are heterogeneous and $E>1$, there exists a sharp phase-transition of the error convergence for not too small stepsizes: The error decays rapidly in the beginning yet later bounces up and stabilizes. 
In addition, provided that the phase-transition time can be estimated, choosing different stepsizes for the two phases can lead to faster overall convergence for both constant and time-decaying stepsizes. {We conjecture that this is because the error in phase 1 is mainly controlled by the initial error with impacting factor decay exponentially in time, and the error in phase 2 is dominated by the collective perturbation caused by environment heterogeneity and multiple local updates (i.e., $E>1$).} 
\end{itemize}

\section{Related Work}
\label{sec: related_work}

{\bf Federated Learning.} 
Federated learning is a communication-efficient distributed machine learning 
approach that enables training global models without sharing raw local data \citep{mcmahan2017communication,kairouz2021advances}.  
Federated learning has been adopted in commercial applications that involve diverse edge devices such as autonomous vehicles \citep{du2020federated,chen2021bdfl,zeng2022federated,posner2021federated,peng2023privacy}, internet of things \citep{nguyen2019diot,yu2020learning}, industrial automation \citep{liu2020fedvision}, healthcare \citep{yan2021experiments,sheller2019multi}, and natural language processing \citep{yang2018applied,ramaswamy2019federated}. 
Multiple open-source frameworks and libraries are available such as 
FATE, Flower, OpenMinded-PySyft, OpenFL, TensorFlow Federated, and NVIDIA Clara. 

FedAvg was proposed in the seminal work \citep{mcmahan2017communication}, and has been one of the most widely implemented federated learning algorithms. It also has inspired many follow-up algorithms such as FedProx \citep{li2020federated}, FedNova \citep{wang2020tackling}, SCAFFOLD \citep{karimireddy2020scaffold}, and adaptive federated methods \citep{deng2020adaptive}.
Despite intensive efforts, the theoretical understanding of FedAvg is far from complete.  
Most existing theoretical work on FedAvg overlooks the underlying data statistics at the agents, which often leads to misalignment of the pessimistic theoretical predictions and empirical success \citep{su2023non,pathak2020fedsplit,Wang_2022_CVPR,wang2022unreasonable}.
This theory and practice gap was studied in a recent work \citep{su2023non} in the context of solving general non-parametric regression problems. 

{\bf Reinforcement Learning.} 
For the single-agent setup, there has been extensive research on the convergence guarantees
of reinforcement learning algorithms. 
A recent surge of work studied non-asymptotic convergence and the corresponding sample complexity. 
\citet{bhandari2018finite} analyzed non-asymptotic Temporal Difference (TD) learning with linear function approximation under a variety of noise conditions, including noiseless, independent noise, and Markovian noise.
The results were extended to TD($\lambda$) and Q-learning.
\citet{li2020sample} investigated the sample complexity of asynchronous Q-learning with different families of learning rates. They also provided an extension of using variance reduction methods inspired by the seminal SVRG algorithm. 
\citet{li2024q} shows the sample complexity of Q-learning. 
Recall that $\calA$ is the set of actions. 
When $|\calA| = 1$, the sample complexity of synchronous Q-learning is sharp and minimax optimal; however, when $|\calA| \ge 2$, it was shown that synchronous Q-learning has a lower bound which is not minimax optimal. 

{\bf Multi-Agent RL.}
{\cite{yu2022surprising} tested multi-agent Proximal Policy Optimization in four multi-agent testbeds wherein agents fully share the parameters, 
and showed its competitive performance. 
\cite{christianos2021scaling} proposed a selective parameter sharing technique, which automatically partitions agents so that they can benefit from the parameter sharing. \cite{Yaodong2024} further proposed provably correct heterogeneous-agent algorithms, which allow agents to have different policy functions. The algorithms showed superior effectiveness and stability in various challenging benchmarks. }

{\bf Federated Reinforcement Learning.} 
\citet{woo_blessing_2023} provided sample complexity guarantees for both synchronous and asynchronous distributed Q-learning.  
They revealed that, under the same transition probability (i.e., homogeneous environment) for all agents, the convergence speed in learning the optimal Q-function can be accelerated linearly in the number of agents.  
They also uncovered the blessing of heterogeneity in terms of state-action exploration -- a completely different notion of heterogeneity from our focus.  
\citet{salgia2025sample} explored the frontier of sample and communication complexities under homogeneous environments. 
Via variance reduction and communication quantization, they designed an algorithm that achieves order-optimal sample and communication complexities.  
\citet{doan2019finite} investigated the distributed TD(0) with linear function approximation for a setting where multiple agents act in a shared environment and each agent has its own reward function. 
\citet{khodadadian2022federated} studied on-policy federated TD learning, off-policy federated TD learning, and federated Q-learning of homogeneous environments and reward with Markovian noises. The sample complexity derived exhibits linear speedup with respect to the number of agents.  

Heterogeneous environments were considered in \citet{jin_federated_2022,wang2023federated,xie_fedkl_2023,zhang2023convergence}.  
\citet{jin_federated_2022} studied federated Q-learning and policy gradient methods assuming known transition probabilities. 
%
To address heterogeneity in both environments and rewards, 
\citet{wang2023federated} proposed FedTD(0) with linear function approximation. 
They proved that, in a low-heterogeneity regime, there is a linear convergence speedup in the number of agents. 
\citet{xie_fedkl_2023} used KL-divergence to penalize the deviation of local update from the global policy, and proved that the local update is beneficial for global convergence. 
\citet{zhang2024finite} proposed FedSARSA using the classic on-policy RL algorithm SARSA with linear function approximation. 
They theoretically proved that the algorithm can converge to a near-optimal solution.
Neither \citet{xie_fedkl_2023} nor \citet{zhang2024finite} characterized sample complexity. 

{\bf Technical comparisons with \cite{woo_blessing_2023,zhang2024finite,wang2023federated}.} \\
\cite{zhang2024finite} and \cite{wang2023federated} examined federated versions of TD learning and SARSA, whereas our paper studied federated Q-learning, which offers distinct theoretical and practical advantages for optimal policy learning.  
Specifically, the upper bound in \cite{zhang2024finite} and \cite{wang2023federated} do not indicate how fundamentally the convergence rates are impacted by the
heterogeneity $\kappa$ and synchronization period $E$. In addition, their upper bounds do not decay to $0$ as $T\rightarrow\infty$. Our upper bound converges to $0$ as $T\rightarrow\infty$. Furthermore, we derived a lower bound on the convergence rates, showing the fundamental limitation of multiple local updates (i.e., $E>1$) in the presence of environmental heterogeneity. To the best of our knowledge, this is the first result of its kind. 

While our analysis of Theorem \ref{thm: sample complexity: synchronous} builds upon the roadmap established by \cite{woo_blessing_2023}, adapting their analysis to our setting introduces significant challenges. In \cite{woo_blessing_2023}, agents operate in homogeneous environments, i.e., each of the \( K \) agents shares the same transition distributions. This homogeneity allows the concentration bound on the difference between the true transition distribution and sampled estimates to become arbitrarily small as the number of samples increases. However, in our setting, each agent has its own environment with a distinct transition distribution. This heterogeneity introduces a perturbation term in the error upper bound that does not decrease with additional samples. Additionally, when \( E > 1 \) and \(\kappa > 0\), the term involving \(\kappa(E - 1)\) in the upper bound becomes the dominant term, resulting in a unique sample complexity. Further technical details and implications of these adjustments are provided in Corollary \ref{cor: sample complex}.

\begin{table} 
\begin{small}
\centering
\begin{tabular}{|m{8em}|m{7.5em}|m{4em}|m{5em}|m{4em}|m{4.3em}|m{4em}|m{3em}|m{3em}|}
\hline
\rowcolor{white}
\textbf{Work} &\textbf{RL\newline Algorithm}& \textbf{Hetero-\newline geneity} & \textbf{Optimality} & \textbf{Lower bound} & \textbf{Sampling} & \textbf{Finite-time}&\textbf{Task}  \\ \hline
\cite{wang2023federated}  &TD(0) & \greencheck & \redxcheck & \redxcheck & \greencheck &\greencheck & Pred \\ \hline
\cite{xie_fedkl_2023} & Policy Gradient & \greencheck & \redxcheck & \redxcheck & \greencheck &\redxcheck & Pred, Plan \\ \hline
\cite{zhang2024finite} &SARSA  & \greencheck & \redxcheck & \redxcheck & \greencheck &\greencheck & Pred, Plan  \\ \hline
\cite{khodadadian2022federated} &TD,\newline Q-Learning & \redxcheck & \greencheck & \redxcheck & \greencheck &\greencheck & Pred, Plan\\ \hline
\cite{jin_federated_2022}&Q-Learning,\newline Policy Gradient & \greencheck & \greencheck & \redxcheck & \redxcheck &\greencheck & Pred, Plan  \\ \hline
\cite{woo_blessing_2023}  & Q-Learning & \redxcheck & \greencheck & \redxcheck & \greencheck &\greencheck & Pred, Plan \\ \hline
\cite{zheng2023federated} & Q-Learning& \redxcheck & \greencheck & \redxcheck & \greencheck &\greencheck & Pred, Plan  \\ \hline
\rowcolor{blue!20}
Our work &Q-Learning & \bluecheck & \bluecheck & \bluecheck & \bluecheck &\bluecheck & Pred, Plan \\ \hline
\end{tabular}
\caption{Comparison of various works in the context of FRL. {Pred and Plan stand for prediction (policy evaluation) and planning (policy optimization), respectively.}}
\label{tbl:comparison}
\vspace{-2em}
\end{small}
\end{table}

\vspace{-.5em}

\section{Preliminary on Q-Learning}
\label{sec: preliminary}
\vspace{-0.5em} 
{\bf Markov decision process.}
A Markov decision process (MDP) is defined by the tuple \(\langle \calS, \calA, \calP, \gamma, R\rangle\), where \(\calS\) represents the set of states, \(\calA\) represents the set of actions, the transition probability \(\calP : \calS \times \calA \rightarrow [0, 1]\) provides the probability distribution over the next states given a current state \(s\) and action \(a\), the reward function \(R : \calS \times \calA \rightarrow [0, 1]\) assigns a reward value to each state-action pair, and the discount factor \(\gamma \in (0, 1)\) models the preference for immediate rewards over future rewards. 
It is worth noting that $\calP =\{P(\cdot \mid s,a)\}_{s\in \calS, a\in \calA}$ is a collection of $|\calS|\times |\calA|$ probability distributions over $\calS$, one for each state-action pair $(s,a)$.  



{\bf Policy, value function, Q-function, and optimality.}
A policy \(\pi\) specifies the action-selection strategy and is defined by the mapping \(\pi : \calS \rightarrow \Delta(A)\), where \(\pi(a \mid s)\) denotes the probability of choosing action \(a\) when in state \(s\). For a given policy \(\pi\), the value function \(V^\pi : \calS \rightarrow \mathbb{R}\) measures the expected total discounted reward starting from state \(s\):
\[
V^\pi(s) = \mathbb{E}_{a_t\sim\pi(\cdot|s_t),s_{t+1}\sim P(\cdot|s_t,a_t)} \left[ \sum_{t} \gamma^t R(s_t, a_t) \mid s_0 = s \right], \quad \forall s \in \calS.
\]
The Q-function (a.k.a.\,state-action value function), \(Q^\pi : \calS \times \calA \rightarrow \mathbb{R}\), evaluates the expected total discounted reward from taking action \(a\) in state \(s\) and then following policy \(\pi\):
\[
Q^\pi(s, a) = R(s, a) \,+\, \mathbb{E}_{a_t\sim\pi(\cdot|s_t),s_{t+1}\sim P(\cdot|s_t,a_t)} \left[ \sum_{t} \gamma^t R(s_t, a_t) \mid s_0 = s, a_0 = a \right], \, \forall (s, a) \in \calS \times \calA.
\]
An optimal policy \(\pi^*\) is one that maximizes the value function for every state, that is \(\forall s\in\calS, V^{\pi^*}(s) \geq V^{\pi}(s)\) for any other \(\pi \ne \pi^*\). Such a policy ensures the highest possible cumulative reward. The optimal value function \(V^*\) (shorthand for \(V^{\pi^*}\)) and the optimal Q-function \(Q^*\) (shorthand for \(Q^{\pi^*}\)) are defined under the optimal policy \(\pi^*\). 

The Bellman optimality equation for the value function and state-value function are:
\begin{align*}
V^*(s) &= \max_a[R(s,a)+ \gamma\sum_{s'\in\calS} P(s'|s,a) V^*(s')] \\
Q^*(s,a) &= R(s,a)+ \gamma\sum_{s'\in\calS}P(s'|s,a) \max_{a'\in \calA} Q^*(s',a'). 
\end{align*}
\paragraph{Q-learning.}
 Q-learning \citep{Watkins1992Qlearning} is a model-free reinforcement learning algorithm that aims to learn the value of actions of all states by updating Q-values through iterative exploration of the environment, ultimately converging to the optimal state-action function. Based on the Bellman optimality equation for the state-action function, the update rule for Q-Learning is formulated as:
\[Q_{t+1}(s,a) = (1-\lambda)Q_t(s,a) +\lambda[R(s,a)+ \gamma\max_{a'\in \calA} Q_t(s',a')], \quad \forall (s,a) \in \calS\times\calA,\]
where \(s'\sim P(\cdot \mid s,a)\),  
and $\lambda$ is the stepsize.

\section{Federated Q-learning}
 
The federated learning system consists of one parameter server (PS) and $K$ agents.  
The $K$ agents are deployed in possibly heterogeneous yet independent environments. 
The $K$ agents are modeled as MDPs 
with $\calM_k = \langle \calS, \calA, \calP^k, \gamma, R\rangle$ for $k=1, \cdots, K$, where $\calP^k = \{P^k(\cdot \mid s, a)\}_{s\in \calS, a\in \calA}$ is the collection of probability distributions. 
In the synchronous setting, each agent $k$ has access to a generative model. 
At each iteration $t$, it generates a new state sample 
%
$s_t^k(s,a) \sim P^k(\cdot \mid s, a)    
$
for each $(s,a)$, 
i.e., 
$\prob{s_t^k(s,a) = s^{\prime}}  =  P^k(s^{\prime}\mid s, a)$ for all $s^{\prime}\in \calS$, independently across state-action pairs $(s,a)$. 
For each $(s,a)$, the global environment $\bar{P}(\cdot \mid s, a)$ \citep{jin_federated_2022}  is defined as 
\begin{align}
\label{eq: global environment}
\vspace{-1em}
\bar{P}(s^{\prime} \mid s, a) =   \frac{1}{K} \sum_{k=1}^K P^k(s^{\prime} \mid s, a), \forall s^{\prime} 
\end{align} 
with the corresponding global MDP defined as $\calM_g=\langle \calS, \calA,  \bar{\calP}, \gamma, R\rangle$.  
Define transition heterogeneity $\kappa$ as 
\begin{align}
\label{eq: transition heterogeneity}
\sup_{k, s, a}\linf{\bar{P}(\cdot \mid s, a) - P^k(\cdot \mid s, a)} := \kappa.  
\end{align} 
Let $Q^*$ denote the optimal Q-function of the global MDP. 
%
 By the Bellman optimality equation, we have, 
\begin{align}
\label{eq: bellman global environment}
Q^*(s,a) = R(s,a)+\gamma \sum_{s^{\prime}\in \calS}\bar P(s^{\prime}\mid s, a) V^*(s^{\prime}), ~~ \forall (s,a)     
\end{align}
where $V^* (s) = \max_{a\in \calA}Q^*(s,a)$ is the optimal value function. 

\renewcommand{\algorithmicrequire}{\textbf{Inputs:}}
\renewcommand{\algorithmicensure}{\textbf{Outputs:}}
 \begin{wrapfigure}[19]{r}{0.5\textwidth}
\vspace{-0.3em}
    \begin{minipage}{0.46\textwidth}
\vspace{-1.8em}
\begin{algorithm}[H]  
\caption{Synchronous Federated Q-Learning}
    \label{alg: synchronous FQL} 
    \begin{algorithmic}[1]
        \Require discount factor $\gamma$, $E$, total iteration $T$, stepsize $\lambda$, initial estimate $Q_0$ 
             \For{$k\in [K]$}
             \State $Q_0^k = Q_0$
             \EndFor   
            \For{$t = 0$ to $T-1$}
                 \For{$k\in [K]$ and $(s,a)\in \calS\times \calA$}
                \State $Q_{t+\frac{1}{2}}^k(s,a) = (1-\lambda) Q_{t}^k (s,a) + \lambda \pth{R(s,a) + \gamma \max_{a^{\prime}\in \calA}\,Q_{t}^k(s_t^k(s,a), a^{\prime})}$ 
                \If{$(t+1)\mod E =0$}
                \State $Q_{t+1}^k =\frac{1}{K}\sum_{k=1}^K Q_{t+\frac{1}{2}}^k$
                \Else 
                \State $Q_{t+1}^k = Q_{t+\frac{1}{2}}^k$
                \EndIf
                 \EndFor
            \EndFor
        \vskip 0.3\baselineskip  
        \State \Return $Q_T = \frac{1}{K}\sum_{k=1}^K Q_T^k$
    \end{algorithmic}
    \label{algorithm:pfedda}
\end{algorithm}
    \end{minipage}
  \end{wrapfigure} 
The goal of federated Q-learning is to have the $K$ agents collaboratively learn $Q^*$. 
We consider synchronous federated Q-learning, which is a natural integration of FedAvg and Q-learning \citep{woo_blessing_2023,jin_federated_2022} -- described in Algorithm \ref{alg: synchronous FQL}.   
Every agent initializes its local $Q^k$ estimate as $Q_0$ and performs standard synchronous Q-learning based on the locally collected samples $s_t^k(s,a)$.  
Whenever $t+1\mod E=0$, through the parameter server, the $K$ agents average their local estimate of $Q$; that is, all agents report their $Q_{t+\frac{1}{2}}^k$ to the parameter server, which computes the average and sends back to agents.

%
%

 \section{Main Results}
With a little abuse of notation, let the matrix $P^k \in \reals^{|\calS||\calA|\times |\calS|}$ represent the transition kernel of the MDP of agent $k$ with the $(s,a)$-th row being $P^k(\cdot \mid s,a)\in \reals^{|\calS|}$ -- the transition probability of the state-action pair $(s,a)$. 
For ease of exposition, we write $P^k(\cdot\mid s,a) = P^k(s, a)$ as the state transition probability at the state-action pair $(s,a)$ when its meaning is clear from the context.  

\subsection{Main Convergence Results.} 
Let $\tilde P_t^k \in \{0,1\}^{|\calS||\calA|\times |\calS|}$ denote the local empirical transition matrix at the $t$-th iteration, defined as 
\begin{align*}
\tilde P_t^k (s^{\prime} \mid s, a)  =\bm{1}\{s^{\prime} = s_t^k(s,a)\}. 
\end{align*}
Denoting $\tilde P_i^kV^* \in \reals^{|\calS||\calA|\times 1}$ with the $(s,a)$-th entry as 
$\tilde P_i^k(s, a)V^* = \sum_{s^{\prime}\in \calS}\tilde P_i^k(s^{\prime}|s, a)V^*(s^{\prime})$. 
Let $\bar Q_{t+1} := \frac{1}{K}\sum_{k=1}^K Q_{t+1}^k$. 
From lines 6, 8, and 10 of Algorithm \ref{alg: synchronous FQL}, it follows that 
\begin{align}
\label{eq: average Q update}
\bar Q_{t+1} = \frac{1}{K}\sum_{k=1}^K\pth{(1-\lambda) Q_{t}^k + \lambda(R+\gamma \tilde{P}^k_t V_t^k)},  
\end{align}
where $V_t^k (s) := \max_{a\in \calA} Q_t^k(s,a)$ for all $s\in \calS$. 
%
Define 
\begin{align}
\label{def: global error}
\Delta_{t+1} := Q^*-\bar Q_{t+1}, ~ \text{and} ~~ \Delta_0 := Q^* - Q_0. 
\end{align}
%
The error iteration $\Delta_t$ is captured in the following lemma. 
\begin{lemma}[Error iteration]
\label{lm: error iteration}
For any $t\ge 0$, 
\begin{align}
\label{eq: error iteration}
\nonumber 
\Delta_{t+1} 
&= (1-\lambda)^{t+1}\Delta_0 +  \gamma\lambda\sum_{i=0}^t (1-\lambda)^{t-i}\frac{1}{K}\sum_{k=1}^K (\bar P - \tilde P_i^k )V^* \\
&\quad + \gamma\lambda\sum_{i=0}^t (1-\lambda)^{t-i}\frac{1}{K} \sum_{k=1}^K\tilde P_i^k(V^*-V_i^k). 
\end{align}    

\end{lemma}
To show the convergence of $\linf{\Delta_{t+1}}$, we bound each of the three terms in the right-hand side of \eqref{eq: error iteration}.  
The following lemma is a coarse error upper bound. 
\begin{lemma}
\label{lm: coarse bound}
Choosing $R(s,a)\in [0,1]$ for each state-action pair $(s,a)$, and choose $0\le Q_0(s,a)\le \frac{1}{1-\gamma}$ for any $(s,a)\in \calS\times \calA$, then $0\le Q_t^k(s,a)\le \frac{1}{1-\gamma}$, 
$0\le Q^*(s,a)\le \frac{1}{1-\gamma}$, 
\begin{align}
\linf{Q^*-Q_t^k} \leq \frac{1}{1-\gamma}, ~~ \text{and} ~  \linf{V^*-V_t^k} \leq \frac{1}{1-\gamma}, ~~~ \forall ~ t\ge 0, \text{and}~ k\in [K].  
\end{align}
\end{lemma}

With the choice of $Q_0$ in Lemma \ref{lm: coarse bound}, the first term in \eqref{eq: error iteration} can be bounded as $\linf{(1-\lambda)^{t+1}\Delta_0} \le (1-\lambda)^{t+1} \frac{1}{1-\gamma}$. 
In addition, as detailed in the proof of Lemma \ref{lm: sample complexity: stepsize push: alternative} and Theorem \ref{thm: sample complexity: synchronous}, the boundedness in Lemma \ref{lm: coarse bound} enables us to bound the second term in \eqref{eq: error iteration} via invoking the Hoeffding's inequality. It remains to bound the third term in \eqref{eq: error iteration}, for which we follow the analysis roadmap of \citet{woo_blessing_2023} by a two-step procedure that is described in Lemma \ref{lm: sample complexity: stepsize push} and Lemma \ref{lm: sample complexity: stepsize push: alternative}. 
Let   
\begin{align}
\label{def: local suboptimal}
\Delta_t^k = Q^*-Q_t^k, ~ \text{and} ~~ \chi(t) = t - (t\mod E), 
\end{align}
i.e., $\Delta_t^k$ is the local error of agent $k$, and $\chi(t)$ is the most recent synchronization iteration of $t$. 
\begin{lemma}
\label{lm: sample complexity: stepsize push}
If $t\mod E = 0$, then $\linf{\frac{1}{K}\sum_{k=1}^K \tilde P_t^k(V^*-V_t)} \le \linf{\Delta_t}$. Otherwise, 
\begin{align*}
\linf{\frac{1}{K}\sum_{k=1}^K \tilde{P}_t^k(V^*-V^k_t)} 
& \le \linf{\Delta_{\chi(t)}} + 2\lambda \frac{1}{K}\sum_{k=1}^K\sum_{t^{\prime} = \chi(t)}^{t-1} \linf{\Delta_{t^{\prime}}^k} \\
& \quad +  \gamma \lambda \frac{1}{K}\sum_{k=1}^K \max_{s, a} \abth{\sum_{t^{\prime} = \chi(t)}^{t-1} \pth{\tilde{P}_{t^{\prime}}^k(s, a) - \bar{P}(s, a) }V^*}. 
\end{align*} 
where we use the convention that $\sum_{t^{\prime} = \chi(t)}^{\chi(t)-1} \linf{\Delta_{t^{\prime}}^k} =0$. 
\end{lemma}

\begin{lemma}
\label{lm: sample complexity: stepsize push: alternative} 
Choose $\lambda\le \frac{1}{E}$. 
For any $\delta \in (0,1)$, with probability at least $(1-\delta)$,  
\begin{align}
\label{eq: bound 111}
\linf{\Delta_{i}^k}
\le \linf{\Delta_{\chi(i)}} + \frac{3\gamma}{1-\gamma}\lambda (E-1) \kappa+ \frac{3\gamma}{1-\gamma}\sqrt{\lambda \log \frac{|\calS||\calA|KT}{\delta}}, \forall ~i \le T, k\in [K].  
\end{align} 
\end{lemma}
%

To bound the $\ell_{\infty}$ norm of the third term in \eqref{eq: error iteration}, we first invoke Lemma \ref{lm: sample complexity: stepsize push}, followed by Lemma \ref{lm: sample complexity: stepsize push: alternative}. It is worth noting that directly applying Lemma \ref{lm: sample complexity: stepsize push: alternative} can also lead to a valid error bound yet the resulting bound will not decay as $T$ increases with proper choice of stepsizes. 

Both Lemma \ref{lm: sample complexity: stepsize push} and Lemma \ref{lm: sample complexity: stepsize push: alternative} are non-trivial adaptations of the approach in \cite{woo_blessing_2023}
 due to the absence of a common optimal action for any given state in heterogeneous environments. Moreover, in the homogeneous setting, each agent draws samples from the same true transition distribution, allowing concentration inequalities to bound the discrepancy between the true distribution and sampled estimates. However, this line of reasoning does not go through in the presence of environmental heterogeneity. 
 When $\kappa>0$, each of the \( K \) agents has its own transition distribution, and the discrepancy is captured by the environmental heterogeneity parameter \( \kappa \).

\begin{theorem}[Convergence]
\label{thm: sample complexity: synchronous}
Choose $E-1\le \frac{1}{\lambda }\min\{\frac{1-\gamma}{4\gamma }, \frac{1}{K}\}$ and $\lambda \le \frac{1}{E}$.  
For any $\delta \in (0, \frac{1}{3})$, with probability at least $1-3\delta$, it holds that 
\begin{align*}
\linf{\Delta_{T}} 
&\leq \frac{4}{(1-\gamma)^2} \exp\sth{-\frac{1}{2}\sqrt{(1-\gamma)\lambda T}} + \frac{14\gamma^2}{(1-\gamma)^2} \lambda(E-1) \kappa +\frac{16}{(1-\gamma)^2}\sqrt{\frac{\lambda}{K} \log \frac{|\calS||\calA|KT}{\delta}}.  
\end{align*}
\end{theorem}

The first term of Theorem \ref{thm: sample complexity: synchronous} is the standard error bound in the absence of environmental heterogeneity and sampling noises.  
The second term arises from environmental heterogeneity. It is clear that when $E=1$, the environmental heterogeneity does not negatively impact the convergence. 
The last term results from the randomness in sampling.  
%
\begin{remark}[Eventual zero error]
It is common to choose the stepsize $\lambda$ based on the time horizon $T$. 
Let $\lambda = g(T)$ be a non-increasing function of $T$, {and other parameters be fixed with respect to $T$.}
As long as $\lambda = g(T)$ decay in $T$, terms 2 and 3 in Theorem \ref{thm: sample complexity: synchronous} will go to 0 as $T$ increases. 
In addition, when $\lambda = \omega(1/T)$, the first term will decay to 0. Conversely, the convergence bounds in \cite{zhang2024finite} and \cite{wang2023federated} do not decay to 0. 
\end{remark} 
There is a tradeoff in the convergence rates of the first term and the remaining terms -- the slower $\lambda$ decay in $T$ leads to faster decay in the first term but slower in the remaining terms.  
Forcing these terms to decay around the same speed leads to slow overall convergence. Corollary \ref{cor: rate} follows immediately from Theorem \ref{thm: sample complexity: synchronous} via carefully choosing $\lambda$ to balance the decay rates of different terms.  
\begin{corollary}
\label{cor: rate}
Choose $(E-1) \le \min \frac{1}{\lambda}\{\frac{1-\gamma}{4\gamma}, \frac{1}{K}\}$, and $\lambda = \frac{4\log^2(TK)}{T(1-\gamma)}$. Let $T\ge E$. For any $\delta \in (0, \frac{1}{3})$, with probability at least $1-3\delta$, 
\begin{align*}
\linf{\Delta_{T}} 
&\leq \frac{4}{(1-\gamma)^2 TK}  
+ \frac{32}{(1-\gamma)^{2.5}} \frac{\log (TK)}{\sqrt{TK}}\sqrt{\log\frac{|\calS||\calA|TK}{\delta}} 
+ \frac{56\log^2(TK)}{(1-\gamma)^3} \frac{E-1}{T} \kappa. 
\end{align*}
\end{corollary}
\begin{remark}
\label{rmk: linear speedup}
Intuitively, both terms 1 and 2 decay as if there are $TK$ iterations. In fact, the decay rate of the sampling noises in Corollary \ref{cor: rate}, with respect to $TK$, is the minimax optimal up to polylog factors \citep{vershynin2018high}. 
The decay of the third term is controlled by environmental heterogeneity when $E>1$. In sharp contrast to the homogeneous settings, larger $E$ significantly slows down the convergence of this term. We show in the next subsection that this slow convergence is fundamental. 
\end{remark}
{\begin{remark}[Communication cost and convergence]
        From Corollary \ref{cor: rate}, by choosing $E = \tilde{\Theta}(\sqrt{T})$, and other parameters are fixed with respect to $T$, we can reach the same error bound of $\tilde\calO(1/\sqrt{T})$ with communication cost of $\tilde{\calO}(\sqrt{T})$, which is better than $\tilde{\calO}(T)$.
    \end{remark}}


\revs{\begin{corollary}
    \label{cor: sample complex}
Choose $E-1\le \frac{1}{\lambda }\min\{\frac{1-\gamma}{4\gamma }, \frac{1}{K}\}$ and $\lambda \le \frac{1}{E}$, and define $x_1 = \frac{4096\log\frac{|\mathcal{S}||\mathcal{A}|K}{\delta}\log^2(\frac{(1-\gamma)^2\epsilon}{8})}{K(1-\gamma)^5\epsilon^2}$,$x_2 = \frac{168\kappa(E-1)\gamma^2 }{(1-\gamma)^3\epsilon}\log^2(\frac{(1-\gamma)^2\epsilon}{12}),$ and $ x_3 = \frac{9216}{K(1-\gamma)^5\epsilon^2}\log\frac{|\mathcal{S}||\mathcal{A}|K}{\delta}\log^2(\frac{(1-\gamma)^2\epsilon}{12}) $,
    \begin{itemize}
        \item When $\kappa=0 \text{ or } E=1,$ for any $\delta \in (0, \frac{1}{3})$, with probability at least $1-3\delta$, it holds that $$\linf{\Delta_T}\leq \epsilon,$$ 
        when 
        $T\geq \exp\{-W_{-1}(-\frac{1}{x_1})\}$,  
        and $\lambda = \frac{\epsilon^2(1-\gamma)^4K}{2304\log\frac{|\calS||\calA|KT}{\delta}}$, $\text{where } W_{-1} \text{ is the Lambert $W$}$  function. The resulting sample complexity is $\tilde\calO\pth{\frac{|\calS||\calA|}{K(1-\gamma)^5\epsilon^2}}$.
        \item When $\kappa>0 \text{ and }E>1,$\begin{itemize}
        \item If $K(E-1)\geq \frac{55\log\frac{|\calS||\calA|K}{\delta}}{\kappa\gamma^2\epsilon(1-\gamma)^2} $, for any $\delta \in (0, \frac{1}{3})$, with probability at least $1-3\delta$, it holds that 

$$\linf{\Delta_T}\leq \epsilon,$$ when 
$T\geq \exp\{-W_{-1}(-\frac{1}{x_2})\}$
and $\lambda= \frac{\epsilon(1-\gamma)^2}{42\kappa(E-1)\gamma^2}$
. The sample complexity is $\tilde\calO\pth{\frac{\kappa|\calS||\calA|E}{(1-\gamma)^3\epsilon}}$.
\item If $K(E-1)\leq  \frac{55\log\frac{|\calS||\calA|K}{\delta}}{\kappa\gamma^2\epsilon(1-\gamma)^2}$, for any $\delta \in (0, \frac{1}{3})$, with probability at least $1-3\delta$, it holds that 

$$\linf{\Delta_T}\leq \epsilon,$$ when 
$T\geq \exp\{-W_{-1}(-\frac{1}{x_3})\}$
and $\lambda = \frac{\epsilon^2(1-\gamma)^4K}{2304\log\frac{|\calS||\calA|KT}{\delta}}$. The sample complexity is $\tilde\calO\pth{\frac{|\calS||\calA|}{K(1-\gamma)^5\epsilon^2}}$.
\end{itemize}
\end{itemize}

\end{corollary}
\begin{remark}[Sample complexity on $K$ and $E$ and conditional linear speedup.] 
        From Corollary \ref{cor: sample complex}, we can conclude that when the setting is homogeneous (\ie, $\kappa = 0$) or the agents communicate every step (\ie, $E=1$), the sample complexity $\tilde\calO\pth{\frac{|\calS||\calA|}{K(1-\gamma)^5\epsilon^2}}$ matches the one in \cite{woo_blessing_2023}. On the other hand, when the setting is heterogeneous (i.e., $\kappa>0$) and $E>1$, it is evident that if the total computation steps per synchronization are sufficiently small, \ie, $K(E-1) \leq \tilde\calO((\kappa\epsilon)^{-1}(1-\gamma)^{-2})$, the sample complexity also matches the one in the homogeneous setting, where there is a linear speedup. Otherwise, the sample complexity $\tilde\calO\pth{\frac{\kappa|\calS||\calA|E}{(1-\gamma)^3\epsilon}}$ 
 increases with $E$, meaning that multiple local rounds only consume more samples (i.e., $E$-times more) samples without achieving linear speedup.
    \end{remark}
}

\subsection{On the Fundamentals of Convergence Slowdown for $E>1$ {in Heterogeneous Environments}.}

\revs{\begin{theorem}
    \label{thm:impossibility}
    Let $Q_0=\bm{0}$.
    For any even $K\ge 2$, there exists a collection of $\{(\calS, \calA, \calP^k, R, \gamma): ~ k\in [K]\}$ where $|\calS|=2$, $|\calA|=1$, and $ R := \begin{bmatrix}
        r_1\\r_2
    \end{bmatrix}$ fixed for all MDPs in the collection, such that, for $E\geq 2$ and time-invariant stepsize $\lambda \le \frac{1}{1+\gamma}$, 
    \[
    \linf{\Delta_{T}} \geq c_R\frac{ E}{(1-\gamma)T},
    \]
    when $T / E \in \naturals$ and $T\geq E \cdot \max\left\{ \exp\{\frac{4E(\gamma+2)}{(1+\gamma)\gamma^2(E-1)}\}, \exp\left\{-W_{-1}\left(-\frac{1-\gamma}{2(1+\gamma)}\right)\right\}\right\}, \text{where } W_{-1} \text{ is the Lambert $W$}$  function, $c_R = \frac{1}{2} \min\left\{\frac{\abth{r_1+r_2}}{e}, \abth{r_1-r_2}  \right\}$ when $r_1\not=r_2$ and $c_R = \frac{\abth{r_1+r_2}}{2e}$ otherwise.
\end{theorem}}


{\bf Proof Sketch.}
Below we provide the proof sketch of \prettyref{thm:impossibility}.  
The full proof is deferred to \prettyref{app: lower bound}.

The eventual slow rate convergence is due to the heterogeneous environments $\calP^k$ regardless of the cardinality of the action space. 
In particular, we prove the slow rate when the action space is a singleton, in which case the Q-function coincides with the V-function. 
The process is also known as the Markov reward process.
According to Algorithm~\ref{algorithm:pfedda}, when $(t+1)\text{ mod } E \ne 0$, we have 
\[
Q_{t+1}^k = \pth{(1-\lambda)I + \lambda \gamma P^k} Q_t^k + \lambda R. 
\]
Following Algorithm~\ref{algorithm:pfedda}, we let \revs{$z$} denote the \revs{$z$}-th synchronization round, and obtain the following recursion between two synchronization rounds:
\begin{equation}
\label{eq:Delta_tE}
\Delta_{(\revs{z}+1)E}
= \bar{A}^{(E)} \Delta_{\revs{z}E} + \pth{ \pth{I- \bar{A}^{(E)} } -   \pth{I+\bar{A}^{(1)}+\dots \bar{A}^{(E-1)}} \pth{I-\bar{A}^{(1)}} } Q^*,
\end{equation}
where $\bar{A}^{(\ell)} \triangleq \frac{1}{K}\sum_{k=1}^K (A^k)^\ell$ and $A^k \triangleq (1-\lambda)I + \lambda \gamma P^k$. 
While the first term on the right-hand side of~\eqref{eq:Delta_tE} decays rapidly to zero, the second term is non-vanishing due to environment heterogeneity for $E\ge 2$. 
Specifically, to ensure the rapid decay of the first term, it is necessary to select a stepsize $\lambda= \tilde\Omega (\frac{1}{\revs{z}E})$.
However, this choice results in the dominating residual error from the second term, which increases linearly with $\lambda E = \tilde\Omega (1/\revs{z})$.

Next, we instantiate the analyses by constructing the set $\calP^k$ over two states and an even number of clients with
\begin{equation}
\label{eq:transition-kernels}
P^{2k-1}=\begin{bmatrix}
    1& 0\\
    0& 1
\end{bmatrix},\quad 
P^{2k}=\begin{bmatrix}
    0& 1\\
    1& 0
\end{bmatrix},\quad \text{for}~k\in\naturals. 
\end{equation}
Applying the formula of $\bar{A}^{(\ell)}$ yields the following eigen-decomposition:
\[
\bar{A}^{(\ell)}= \alpha_\ell (I-\bar P) + \beta_\ell \bar P,
\]
where $\bar P = \frac{1}{2}\mathbf{11}^\top$, $\alpha_\ell \triangleq \frac{1}{2}(\nu_1^\ell + \nu_2^\ell)$, $\beta_\ell \triangleq \nu_2^\ell$, $\nu_1\triangleq 1-(1+\gamma)\lambda$, and $\nu_2\triangleq 1-(1-\gamma)\lambda$.
For this instance of $\calP_k$, the error evolution~\eqref{eq:Delta_tE} reduces to $\Delta_{(\revs{z}+1)E}= \pth{\alpha_E (I-\bar P) + \beta_E \bar P} \Delta_{\revs{z}E}
+ \kappa_E (I-\bar P) Q^*$ with 
$\kappa_E \triangleq - \frac{\gamma}{2}\pth{\frac{1-\nu_2^E}{1-\gamma}- \frac{1-\nu_1^E}{1+\gamma}}$,
which further yields the following full error recursion: 
\[
\Delta_{\revs{z}E}= \pth{\alpha_E^{\revs{z}} (I-\bar P) + \beta_E^{\revs{z}} \bar P} \Delta_{0} + \frac{1-\alpha_E^{\revs{z}}}{1-\alpha_E}\kappa_E (I-\bar P) Q^*.
\]

Starting from $Q_0=\bm{0}$, the error can be decomposed into
\begin{equation}
\label{eq:Delta_tE-zero-start}
\Delta_{\revs{z}E}= \beta_E^z \bar P Q^* + \pth{\alpha_E^{\revs{z}} + \frac{1-\alpha_E^{\revs{z}}}{1-\alpha_E}\kappa_E  } (I-\bar P)Q^*.
\end{equation}
The two terms of the error are orthogonal and both non-vanishing.
Therefore, it remains to lower bound the maximum magnitude of two coefficients irrespective of the stepsize $\lambda$.

{To this end, we analyze two regimes of $\lambda$ separated by a threshold $\lambda_0 \triangleq \frac{\log (T/E)}{(1-\gamma) T}$:}  
\begin{itemize}
\item Slow rate due to small stepsize when $\lambda\le \lambda_0 $. Since $\beta_E^{\revs{z}}$ decreases as $\lambda$ increases,
\[
\beta_E^{\revs{z}} \geq  (1-(1-\gamma)\lambda_0)^{\revs{z}E} = \pth{1-\frac{\log \revs{z}}{\revs{z}E} }^{\revs{z}E}{\geq \frac{E}{eT}}.
\]
\item Slow rate due to environment heterogeneity when $\lambda \ge \lambda_0$. We show that
{\[\left|{\alpha_E^{\revs{z}} + \frac{1-\alpha_E^{\revs{z}}}{1-\alpha_E}\kappa_E }\right|
\ge
\frac{E}{(1-\gamma)T}.\]}
\end{itemize}
We conclude that at least one component of the error in~\eqref{eq:Delta_tE-zero-start} must be slower than the rate $\Omega(E/T)$.

\begin{remark}
The explicit calculations are based on a set $\calP^k$ over a pair of states. 
Nevertheless, the evolution~\eqref{eq:Delta_tE} is generally applicable. 
Similar analyses can be extended to scenarios involving more than two states, provided that the sequence of matrices $\bar{A}^{(\ell)}$ is simultaneously diagonalizable. 
For instance, the construction of the transition kernels in~\eqref{eq:transition-kernels} can be readily extended to multiple states if the set $\calS$ can be partitioned into two different classes.
The key insight is the non-vanishing residual on the right-hand side of~\eqref{eq:Delta_tE} when $E\ge 2$ due to the environment heterogeneity.
\end{remark}

\subsection{Discussion on Time-varying Stepsize}

Although using time-varying stepsize is common and simple when implementing the algorithm, it is not easy to transfer from current time-invariant stepsize analysis to time-varying stepsize analysis. This is because in the time-invariant stepsize analysis we are dealing with a function of one variable, however, in the time-varying case, we are dealing with a function of $T$ variables.

For example, in our lower bound analysis, we picked a threshold $\lambda_0$ and showed that no matter $\lambda$ is greater or smaller than $\lambda_0$, the convergence rate is greater than {$\Theta_R\pth{E/((1-\gamma)T)}$}, and we can claim we have covered all the cases. However, for time-varying stepsize, the number of stepsizes is $T$, and it is not easy to generalize a similar result by just considering several cases because each stepsize gives an additional dimension. Even if we know the sequence is decaying, without specifying a particular family of stepsizes, it is not possible to divide it into several cases as we did for time-invariant stepsize.

We conjecture that both approaches lead to comparable residual error levels over extended training. For example, the stepsizes used in Figure \ref{fig:two-phase a} and Figure \ref{fig:two-phase b} are $\frac{1}{\sqrt{T}}$ and $\frac{1}{\sqrt{t+1}}$, respectively. While the time-decaying stepsize appears to have faster initial convergence due to its larger values, we observe that as \( t \) increases, the convergence rates of the two strategies seem to align, suggesting a similar asymptotic behavior.





\vspace*{-\baselineskip}
\section{Experiments}
\vspace*{-.5\baselineskip}

%
{\bf Description of the setup.} 
%
In our experiments, we consider $K=20$ agents \citep{jin_federated_2022}, each interacting with an independently and randomly generated \(5\times5\) maze environment $\langle \calS, \calA, \calP^k, R, \gamma\rangle$ for $k\in \{1,2,\cdots, 20\}$.   
The state set $\calS$ contains $25$ cells that the agent is currently in.  
The action set contains 4 actions $\calA = \{\text{left}, \text{up}, \text{right}, \text{down}\}$. Thus, $|\calS|\times |\calA| =100$. 
We choose $\gamma = 0.99$. 
For ease of verifying our theory, each entry of the reward \(R\in \reals^{100}\) is 
sampled from $\Bern(p=0.05)$, which slightly departs from a typical maze environment wherein only two state-action pairs have nonzero rewards. 
We choose this reward so that $\linf{\Delta_0}\approx 100 = \frac{1}{1-\gamma}$, which is the coarse upper bound of $\linf{\Delta_t}$ for all $t$.  
For each agent $k$, its state transition probability vectors $\calP^k$ are constructed on top of standard state transition probability vectors of maze environments incorporated with a drifting probability $0.1$ in each non-intentional action as in \texttt{WindyCliff} \citep{jin_federated_2022, paul2019fingerprint}. In this way, the environment heterogeneity lies not only in the differences of the non-zero probability values \citep{jin_federated_2022,paul2019fingerprint} but also in the probability supports (i.e., the locations of non-zero entries).  
Our construction is more challenging: The environment heterogeneity \(\kappa\) as per (\ref{eq: transition heterogeneity}) of our environment construction was calculated to be $1.2$. Yet, the largest environment heterogeneity of the \texttt{WindyCliff} construction in \citet{jin_federated_2022} is about 0.31.

%
%
%
%

We choose \(Q_0 = \bold{0}\in \reals^{100}\). 
All numerical results are based on 5 independent runs to capture the variability. The dark lines represent the mean of the runs, while the shaded areas around each line illustrate the range obtained by adding and subtracting one standard deviation from the mean. The maximum time duration is $T=20,000$ in our experiment since it is sufficient to capture the characteristics of the training process.

\paragraph{Convergence behavior and two-phase phenomenon.}

We demonstrate through numerical simulations that our analysis aligns with the observed behaviors. For algorithms with a time-invariant stepsize, convergence requires sufficiently small stepsizes and a sufficiently large number of iterations $T$.

To explore the impact of stepsizes on convergence, we use $\lambda \in \{0.9, 0.5, 0.2, 0.1, 0.05\}$, spanning a range within $(0,1)$. As shown in Figure \ref{fig:two-phase-a}, these stepsizes are not sufficiently small, leading to a two-phase phenomenon: the $\ell_{\infty}$-norm of $\Delta_t = Q^* - \bar{Q}_t $ has a rapid decay in the first phase followed by a bounce back in the second phase. This phenomenon is distinctive to heterogeneous settings. In contrast, Figure \ref{fig:two-phase-b} indicates that in homogeneous environments, no drastic bounce occurs, irrespective of the stepsize. \revs{Note that if multiple plots on the same page share the same legend, we display the legend only once for clarity.}

Figure \ref{fig:two-phase a} (light blue curve) demonstrates that with a sufficiently small stepsize, such as $\lambda = \frac{1}{\sqrt{T}}$, the error continuously decreases, reaching approximately 24 by iteration 20,000.

A useful practice implication of our results is that: 
While constant stepsizes are often used in reinforcement learning problems because of the great performance in applications as described in \citet{sutton2018_chap2}, 
they suffer significant performance degradation in the presence of environmental heterogeneity. 

\begin{figure}[h]
  \begin{subfigure}{0.475\textwidth}
    \includegraphics[width=\textwidth]{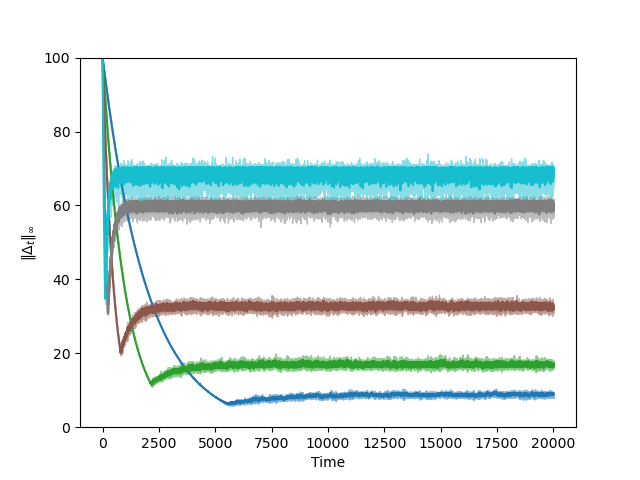} 
    \caption{Heterogeneous environments \(E=10\).}
    \label{fig:two-phase-a}
  \end{subfigure}
  \hfill
  \begin{subfigure}{0.475\textwidth}
    \includegraphics[width=\textwidth]
    {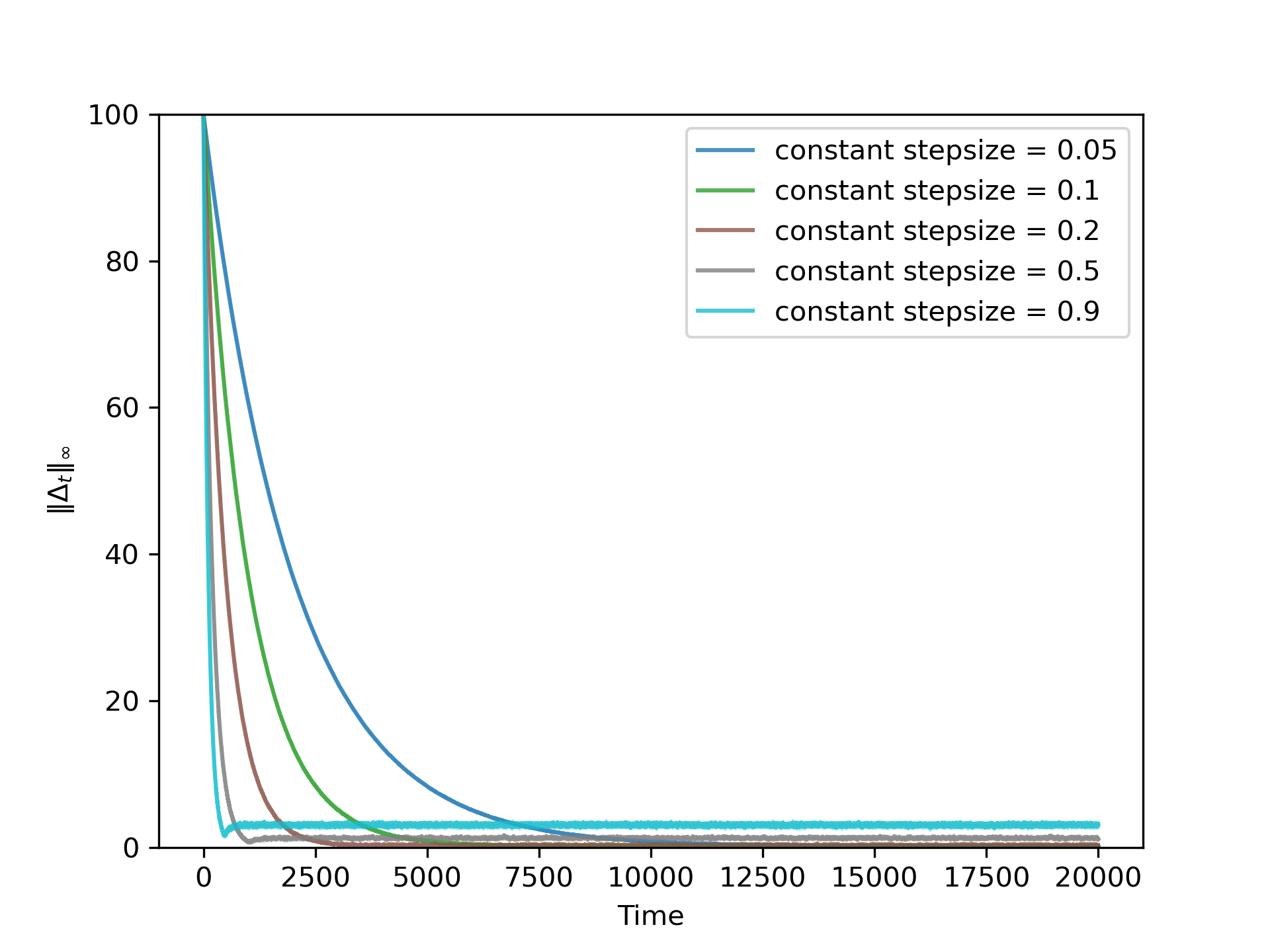}
    \caption{Homogeneous environments \(E=10\).}
    \label{fig:two-phase-b}
  \end{subfigure} 
  \caption{The \(\ell_{\infty}\) error of different constant stepsizes under the heterogeneous and homogenous settings. }
  \label{fig: homo_vs_hetero}
\end{figure}

\vspace{0.5em}

\paragraph{Impacts of the synchronization period $E$.}

In homogeneous settings (refer to Figure \ref{fig:homo_sync_int} in Appendix \ref{app: exp: homo: E}), the synchronization period $E$ has negligible impact, consistent with prior findings in the literature \citep{woo_blessing_2023,khodadadian2022federated}. However, under heterogeneous conditions, larger $E$ values lead to increased final error across the five constant stepsizes, as depicted in Figure \ref{fig:sync_int} and Figure \ref{fig:two-phase-a}. This degradation persists even with time-decaying stepsizes $\lambda_t = \frac{1}{\sqrt{t+1}}$, as shown in Figure \ref{fig:diffE}. We hypothesize that larger $E$ values require either smaller or more rapidly decaying stepsizes to mitigate the degradation caused by increased synchronization periods.

\paragraph{Potential utilization of the two-phase phenomenon.}
As shown in Figures \ref{fig:two-phase-a} and \ref{fig:sync_int}, in the presence of environmental heterogeneity, the smaller the stepsizes, the smaller error $\linf{\Delta_t}$ can reach and less significant of the error bouncing in the second phase. In our preliminary experiments, we tested small stepsizes $\lambda = 1/T^{\alpha}$ for $\alpha \in \{0.4, 0.5, \cdots, 1\}$, which eventually lead to small errors yet at the cost of being extremely slow. Among these choices, $\lambda = 1/\sqrt{T}$ has the fastest convergence performance yet is still $\approx 24$ up to iteration 20,000.

\begin{figure}[H]
\vspace{-1.8em}
  \begin{subfigure}[t]{0.475\textwidth}
    \includegraphics[width=\textwidth]{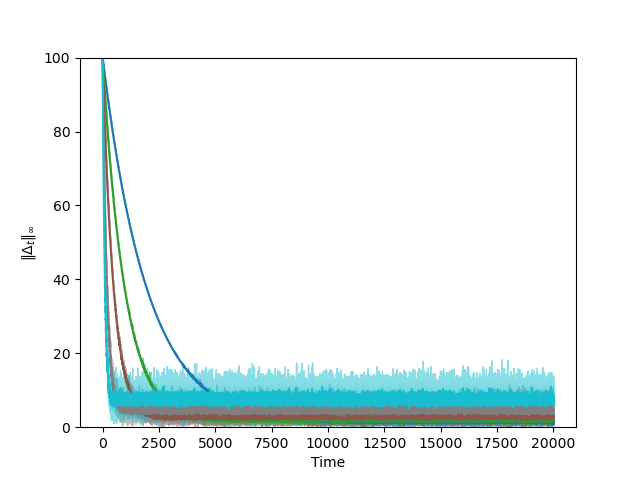}
    \caption{E=1}
    \label{fig-a}
  \end{subfigure}
  \hfill
  \begin{subfigure}[t]{0.475\textwidth}
    \includegraphics[width=\textwidth]{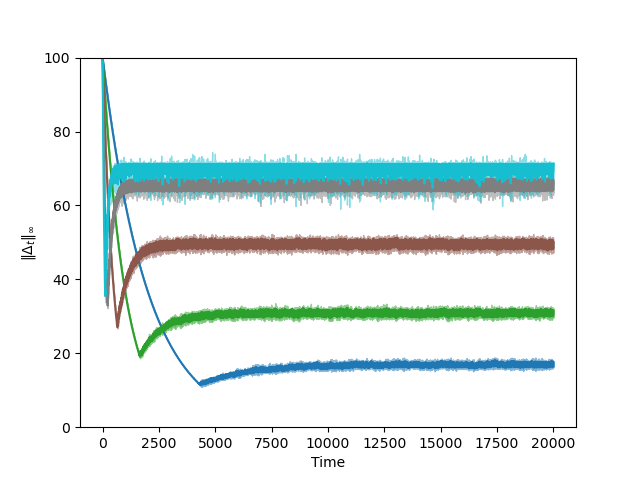}
    \caption{E=20}
    \label{fig-a}
  \end{subfigure}
\vspace{-1em}
  \begin{subfigure}[t]{0.475\textwidth}
    \includegraphics[width=\textwidth]{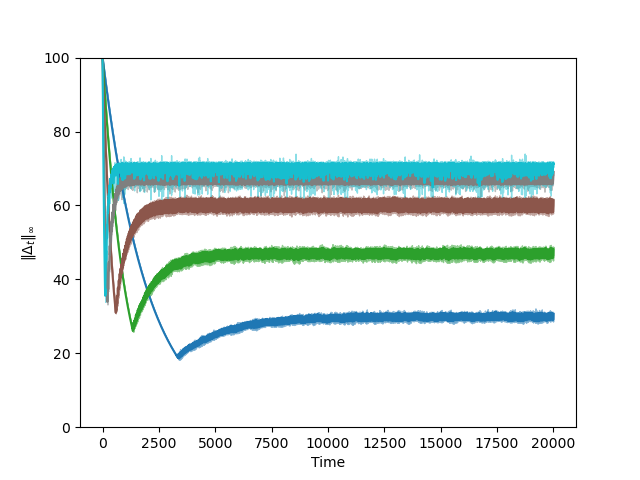}
    \caption{E=40}
    \label{fig-b}
  \end{subfigure}
   \hfill
  \begin{subfigure}[t]{0.475\textwidth}
    \includegraphics[width=\textwidth]{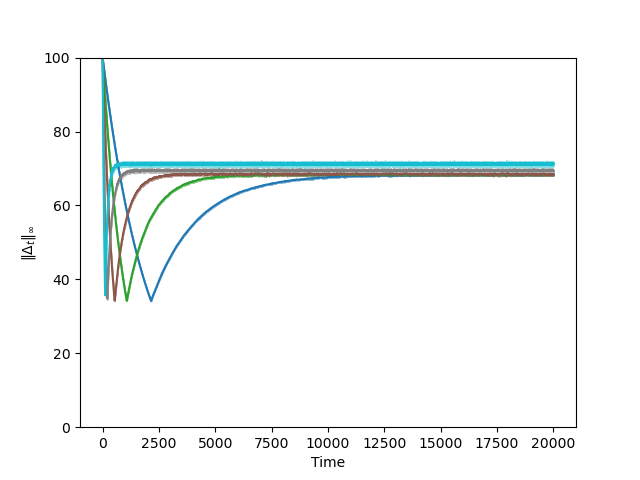}
    \caption{E=\(\infty\)}
    \label{fig-b}
  \end{subfigure}

  \caption{{
  Convergence behavior for constant stepsizes $(0.05,0.1,0.2,0.5,0.9)$ under various synchronization intervals $E$ $(1, 20, 40, \infty)$. In heterogeneous settings, higher $E$ and larger $\lambda$ lead to higher residual errors.}
}
\vspace{-0.5em}
  \label{fig:sync_int}
\end{figure}



Let $t_0$ be the iteration at which the error trajectory $\linf{\Delta_t}$ switches from phase 1 to phase 2. 
Provided that $t_0$ can be estimated, choosing different stepsizes for the two phases can lead to faster overall convergence, compared with using the same stepsize throughout. 

Figure \ref{fig:two-phase a} illustrates two-phase training with different phase 1 stepsizes and phase 2 stepsize $\lambda = 1/\sqrt{T}$ compared with using $\lambda = 1/\sqrt{T}$ throughout. 
Overall, using $\lambda = 1/\sqrt{T}$ throughout leads to the slowest convergence, highlighting the benefits of the two-phase training strategy. 
Among all two-phase stepsize choices,  the stepsize of 0.05 in the first phase results in a longer phase 1 duration (\(t_0 = 5550\)) but the lowest final error (2.75327), suggesting a better convergence. We further test the convergence performance with respect to different target error levels, details can be found in Appendix \ref{app: two-phase: stepsizes}. 

{We also evaluated the two-phase training strategy using various time-decaying step sizes, including $\frac{1}{\sqrt{t+1}}$, $\frac{c+1}{t+c}$, $\frac{1}{t+1}$, and $\frac{1}{(t+1)^{0.7}}$. In all cases, Figure \ref{fig:two-phase} shows the two-phase training has an advantage.
}

We leave the estimation and characterization of $t_0$ for future work.

%

\label{sec:experiments}



\section*{Acknowledgments}
We thank He Cheng for his valuable assistance on  the experiments.
P.\ Yang is supported in part by the National Key R\&D Program of China 2024YFA1015800, and Tsinghua University Dushi Program 2025Z11DSZ001. 
L.\ Su is supported in part by NSF CAREER CIF-2340482. 

\bibliography{main}
\bibliographystyle{tmlr}


\appendix

\newpage 
\appendix 
\begin{center}
    \Large Appendices 
\end{center}

\section{Proof of Lemma \ref{lm: error iteration}}

The evolution of $\Delta_{t+1}$ can be decomposed as follows:  
\begin{align*}
	\Delta_{t+1}&=Q^*-\bar Q_{t+1} \\
 &\overset{(a)}{=} \frac{1}{K}\sum_{k=1}^K(Q^* - ((1-\lambda)Q_t^k+\lambda(R+\gamma \tilde P_t^kV_t^k)))\\
	& = \frac{1}{K}\sum_{k=1}^K((1 - \lambda)(Q^*-Q_t^k) + \lambda(Q^*- R-\gamma \tilde P_t^k V_t^k))\\
	& \overset{(b)}{=} (1 - \lambda)\Delta_{t} + \gamma \lambda\frac{1}{K}\sum_{k=1}^K(\bar PV^*- \tilde P_t^k V_t^k)\\
	& = (1 - \lambda)\Delta_{t} + \frac{\gamma \lambda}{K}\sum_{k=1}^K (\bar P - \tilde P_t^k )V^*+\frac{\gamma \lambda}{K}\sum_{k=1}^K \tilde P_t^k (V^* - V_t^k) \\
        &=(1-\lambda)^{t+1}\Delta_0 +  \gamma\lambda\sum_{i=0}^t (1-\lambda)^{t-i}\frac{1}{K}\sum_{k=1}^K (\bar P - \tilde P_i^k )V^* \\
&\quad + \gamma\lambda\sum_{i=0}^t (1-\lambda)^{t-i}\frac{1}{K} \sum_{k=1}^K\tilde P_i^k(V^*-V_i^k),  
\end{align*}
where equality (a) follows from Eq.\,\eqref{eq: average Q update}, equality (b) follows from the Bellman optimality equation in Eq.\,\eqref{eq: bellman global environment}, and the last equality follows from unrolling the updates $t+1$ times and using the fact that $\Delta_0 = Q^* - Q_0$.

\section{Proof of Lemma \ref{lm: coarse bound}}
We first show $0 \leq Q_t^k(s,a) \leq \frac{1}{1-\gamma}$ by inducting on $t$. 
When $t=0$, this is true by the choice of $Q_0$. 
%
Suppose that $0 \leq Q_{t-1}^k(s,a) \leq \frac{1}{1-\gamma}$ for any state-action pair $(s,a)$ and any client $k$. 
Let's focus on time $t$. When $t$ is not a synchronization iteration (i.e., {$t+1\mod E \not=0$}), we have 
\begin{align*}
    Q_t^k (s,a) &= (1-\lambda)Q_{t-1}^k (s,a)+\lambda(R(s,a)+\gamma \tilde P_t^k(s,a) V_{t-1}^k) \\
    & \leq \frac{1-\lambda}{1-\gamma}+\lambda (R(s,a)+\gamma \tilde P_t^k(s,a) V_{t-1}^k)\\
    & \stackrel{(a)}{\leq} \frac{1-\lambda}{1-\gamma}+\lambda (1+\frac{\gamma}{1-\gamma})\\
    & \leq \frac{1}{1-\gamma}-\frac{\lambda}{1-\gamma}+\frac{\lambda }{1-\gamma}\\
    &=\frac{1}{1-\gamma}, 
\end{align*}
where inequality (a) holds because for any $s, V_{t-1}^k(s) = \max_{a \in \calA} Q_{t-1}^k(s,a)\leq \frac{1}{1-\gamma}$ by the inductive hypothesis, {and each element of $\tilde{P}_t^k(s,a) \in [0,1]$. Then $\tilde{P}_t^k(s,a)V_{t-1}^k\leq \|\tilde{P}_t^k(s,a)\|_1\linf{V_{t-1}^k}\leq \frac{1}{1-\gamma}$ by  H\"older's inequality.} Similarly, we can show the case when $t$ is a synchronization iteration.

With the above argument, we can also show that 
$0\le Q^*(s,a)\le \frac{1}{1-\gamma}$ for any state-action pair $(s,a)$.  Therefore, we have that $\linf{Q^*-Q_t^k} \leq \frac{1}{1-\gamma}$. 

Next, we show that bound on $\linf{V^*- V_t^k}$. 
\begin{align*}
    \linf{V^*- V_t^k} &=\max_{s \in \calS} \abth{V^*(s) - V_t^k(s)} \\
    &= \max_{s \in \calS} \abth{ \max_{a \in \calA}Q^*(s,a) - \max_{a'\in \calA}Q_t^k(s,a')}\\
    &\le \max_{s \in \calS,a \in \calA} \abth{Q^*(s,a) - Q_t^k(s,a)}\\
    &= \linf{Q^*-Q_t^k} \\
    &\leq \frac{1}{1-\gamma}. 
\end{align*}

\section{Proof of Lemma \ref{lm: sample complexity: stepsize push}}
%
When $t\mod E = 0$, i.e., $i$ is a synchronization round, $Q_t^k = Q_t^{k^{\prime}}$ for any pair of agents $k, k^{\prime}\in [K]$. Hence,  
\begin{align}
\label{eq: E22 synchornization bound}
\nonumber
\frac{1}{K}\sum_{k=1}^K \tilde P_t^k(s,a)(V^*-V_t^k)
& = \pth{\frac{1}{K}\sum_{k=1}^K \tilde P_t^k(s,a)}(V^*-\bar V_t)\\
\nonumber
& \le  \|\frac{1}{K}\sum_{k=1}^K \tilde P_t^k(s,a)\|_1 \linf{V^* - \bar V_t}\\
\nonumber 
&\le \linf{V^* - \bar V_t}\\
\nonumber 
& \leq \linf{Q^* - \bar Q_t} \\
& = \linf{\Delta_t}.  
\end{align}
For general $t$, we anchor the error term to that of the synchronization round as follows:
\begin{align}
\label{eq: E22 general round}
\nonumber  
    \linf{\frac{1}{K}\sum_{k=1}^K \tilde{P}_t^k(V^*-V^k_t)} 
    &= \linf{\frac{1}{K}\sum_{k=1}^K \tilde{P}_t^k(V^*-V_{\chi(t)}^k + V_{\chi(t)}^k - V^k_i)}\\
\nonumber  
    &\leq \linf{\frac{1}{K}\sum_{k=1}^K \tilde{P}_t^k(V^*-V_{\chi(t)}^k)}+ \linf{\frac{1}{K}\sum_{k=1}^K \tilde{P}_t^k( V_{\chi(t)}^k - V^k_t)} \\
    \nonumber 
    & \overset{(a)}{\le} \linf{\Delta_{\chi(t)}} + \linf{\frac{1}{K}\sum_{k=1}^K \tilde{P}_t^k( V_{\chi(t)}^k - V^k_t)} \\
    & \le \linf{\Delta_{\chi(t)}} + \frac{1}{K}\sum_{k=1}^K \linf{V_{\chi(t)}^k - V^k_t}. 
\end{align}
where inequality (a) follows from Eq.\,\eqref{eq: E22 synchornization bound}. 
For any state $s\in \calS$, let $a_t^k(s) \in \arg\max_{a\in \calA} Q_t^k(s,a)$ for all $t$ and $k$.  
We have  
\begin{align}
\label{eq: perturbation within a round: sample}
\nonumber 
&V^k_t(s) - V_{\chi(t)}^k(s)\\ 
\nonumber 
&= Q^k_t(s, a_t^k(s)) -  Q_{\chi(t)}^k(s, a_{\chi(t)}^k(s)) \\
\nonumber
& \overset{(a)}{\le} Q^k_t(s, a_t^k(s)) -  Q_{\chi(t)}^k(s, a_t^k(s)) \\
\nonumber 
& =  Q^k_t(s, a_t^k(s)) -  Q^k_{t-1}(s, a_t^k(s)) + Q^k_{t-1}(s, a_t^k(s)) - Q^k_{t-2}(s, a_t^k(s)) \\ 
& \quad + \cdots  
 + Q^k_{\chi(t)+1}(s, a_t^k(s)) -  Q_{\chi(t)}^k(s, a_t^k(s)).  
\end{align}
where inequality (a) holds because \(Q_{\chi(t)}^k(s, a_t^k(s)) \leq  Q_{\chi(t)}^k(s, a_{\chi(t)}^k(s))\).
%
For each $t^{\prime}$ such that $\chi(t) \le t^{\prime}\le t$, it holds that,  
\begin{align*}
&Q^k_{t^{\prime}+1}(s, a_t^k(s)) -  Q_{t^{\prime}}^k(s, a_t^k(s))  \\ 
& = (1-\lambda) Q^k_{t^{\prime}}(s, a_t^k(s)) + \lambda (R(s, a_t^k(s)) + \gamma \tilde{P}_{t^{\prime}}^k(s, a_t^k(s))V_{t^{\prime}}^k)-  Q_{t^{\prime}}^k(s, a_t^k(s)) \\
& \overset{(a)}{=} -\lambda Q^k_{t^{\prime}}(s, a_t^k(s)) 
+ \lambda \pth{ Q^*(s, a_t^k(s)) - R(s, a_t^k(s)) - \gamma \bar{P}(s, a_t^k(s)) V^* + R(s, a_t^k(s)) + \gamma \tilde{P}_{t^{\prime}}^k(s, a_t^k(s))V_{t^{\prime}}^k} \\
& = \lambda \Delta_{t^{\prime}}^k(s, a_t^k(s))
 + \gamma \lambda \pth{\tilde{P}_{t^{\prime}}^k(s, a_t^k(s))V_{t^{\prime}}^k - \bar{P}(s, a_t^k(s)) V^*}\\
& = \lambda \Delta_{t^{\prime}}^k(s, a_t^k(s))
 + \gamma \lambda \pth{(\tilde{P}_{t^{\prime}}^k(s, a_t^k(s)) - \bar{P}(s, a_t^k(s)))V^* + \tilde{P}_{t^{\prime}}^k(s, a_t^k(s)) (V_{t^{\prime}}^k- V^*)} \\
& \le 2 \lambda \linf{\Delta_{t^{\prime}}^k} + \gamma \lambda \pth{\tilde{P}_{t^{\prime}}^k(s, a_t^k(s)) - \bar{P}(s, a_t^k(s))}V^*, 
\end{align*}
where equality (a) follows from the Bellman equation Eq.\,\eqref{eq: bellman global environment}, and the last inequality follows from Eq.\eqref{eq: E22 synchornization bound} and that  
\begin{align*}
\gamma \lambda  \tilde{P}_{t^{\prime}}^k(s, a_t^k(s)) (V_{t^{\prime}}^k- V^*)  
\le \gamma  \lambda  \|\tilde{P}_{t^{\prime}}^k(s, a_t^k(s))\|_1 \|V_{t^{\prime}}^k- V^*\|_{\infty} 
\le \gamma\|V_{t^{\prime}}^k- V^*\|_{\infty}.
\end{align*}
Thus, $V^k_t(s) - V_{\chi(t)}^k(s)$ can be upper bounded as 
\begin{align} 
\label{eq: upper bound: perturbation: within a round: sample}
\nonumber 
V^k_t(s) - V_{\chi(t)}^k(s)
&\le \sum_{t^{\prime} = \chi(t)}^{t-1} Q^k_{t^{\prime}+1}(s, a_t^k(s)) -  Q_{t^{\prime}}^k(s, a_t^k(s))\\
& = 2\lambda \sum_{t^{\prime} = \chi(t)}^{t-1} \linf{\Delta_{t^{\prime}}^k}
+ \gamma \lambda \sum_{t^{\prime} = \chi(t)}^{t-1} \pth{\tilde{P}_{t^{\prime}}^k(s, a_t^k(s)) - \bar{P}(s, a_t^k(s))}V^*. 
\end{align}
Similarly, we have the following lower bound:
\begin{align} 
\label{eq: lower bound: perturbation: within a round: sample}
\nonumber 
V^k_t(s) - V_{\chi(t)}^k(s)
&\ge \sum_{t^{\prime} = \chi(t)}^{t-1} Q^k_{t^{\prime}+1}(s, a_{\chi(t)}^k(s)) -  Q_{t^{\prime}}^k(s, a_{\chi(t)}^k(s))\\
& \ge - 2\lambda \sum_{t^{\prime} = \chi(t)}^{t-1} \linf{\Delta_{t^{\prime}}^k}
+ \gamma \lambda \sum_{t^{\prime} = \chi(t)}^{t-1} \pth{\tilde{P}_{t^{\prime}}^k(s, a_{\chi(t)}^k(s)) - \bar{P}(s, a_{\chi(t)}^k(s))}V^*. 
\end{align}
Plugging the bounds in Eq.\,\eqref{eq: upper bound: perturbation: within a round: sample} and in Eq.\,\eqref{eq: lower bound: perturbation: within a round: sample} back into Eq.\,\eqref{eq: E22 general round}, we get 
\begin{align*}
\linf{\frac{1}{K}\sum_{k=1}^K \tilde{P}_t^k(V^*-V^k_t)} 
& \le \linf{\Delta_{\chi(t)}} + \frac{1}{K}\sum_{k=1}^K \linf{V_{\chi(t)}^k - V^k_t}\\
& \le \linf{\Delta_{\chi(t)}} + 2\lambda \frac{1}{K}\sum_{k=1}^K\sum_{t^{\prime} = \chi(t)}^{t-1} \linf{\Delta_{t^{\prime}}^k} \\
& \quad +  \gamma \lambda \frac{1}{K}\sum_{k=1}^K \max_{s, a} {\abth{\sum_{t^{\prime} = \chi(t)}^{t-1} \pth{\tilde{P}_{t^{\prime}}^k(s, a) - \bar{P}(s, a) }V^*}}, 
\end{align*}
proving the lemma.

\section{Proof of Lemma \ref{lm: sample complexity: stepsize push: alternative}}
When $i\mod E=0$, then $\Delta_{i}^k = \Delta_{\chi(i)}$. 
When $i\mod E\not=0$, we have 
\begin{align*}
Q_{i}^k &= (1-\lambda)Q_{i-1}^k+
\lambda \pth{R + \gamma \tilde{P}_{i-1}^k V_{i-1}^k} \\
& = (1-\lambda)Q_{i-1}^k+
\lambda \pth{Q^* - R - \gamma \bar{P} V^* + R + \gamma \tilde{P}_{i-1}^k V_{i-1}^k}. 
\end{align*}
So, 
\begin{align} 
\label{eq: error iterate within a round}
\nonumber
\Delta_{i}^k 
& = (1-\lambda)\Delta_{i-1}^k+
\lambda \gamma\pth{\bar{P} V^* - \tilde{P}_{i-1}^k V_{i-1}^k} \\
\nonumber 
& = (1-\lambda)\Delta_{i-1}^k+
\lambda \gamma(\bar{P} - \tilde{P}_{i-1}^k) V^* + \lambda \gamma\tilde{P}_{i-1}^k (V^* - V_{i-1}^k) \\
\nonumber 
& \le (1-\lambda)^{i-\chi(i)}\Delta_{\chi(i)}+ \gamma \lambda \sum_{j=\chi(i)}^{i-1} (1-\lambda)^{i-j-1}(\bar{P} - \tilde{P}_j^k)V^* \\
& \quad + \gamma \lambda \sum_{j=\chi(i)}^{i-1} (1-\lambda)^{i-j-1} \tilde{P}_j^k (V^* - V_j^k). 
\end{align}  
For any state-action pair $(s,a)$, 
\begin{align}
\label{eq: error: within a round: term 1}
|(1-\lambda)^{i-\chi(i)}\Delta_{\chi(i)}(s,a)|\le (1-\lambda)^{i-\chi(i)} \linf{\Delta_{\chi(i)}}. 
\end{align}
{For the second term, we have
\begin{align}
\label{eq: error: within a round: term 2}
&\linf{\gamma \lambda \sum_{j=\chi(i)}^{i-1} (1-\lambda)^{i-j-1}(\bar{P} - \tilde{P}_j^k)V^*} \\
\label{eq: error: within a round: term2: decomposed}
\leq& \linf{\gamma \lambda \sum_{j=\chi(i)}^{i-1} (1-\lambda)^{i-j-1}(\bar{P} - {P}_j^k)V^*}+\linf{\gamma \lambda \sum_{j=\chi(i)}^{i-1} (1-\lambda)^{i-j-1}({P}_j^k - \tilde{P}_j^k)V^*} \\
\label{eq: error: within a round: term2: hoeff}
\le& \frac{\gamma}{1-\gamma}\lambda \sum_{j=\chi(i)}^{i-1} (1-\lambda)^{i-1-j} \kappa  
 + \frac{\gamma}{1-\gamma}\sqrt{\lambda \log \frac{|\calS||\calA|KT}{\delta}} \\
  \le& \frac{\gamma}{1-\gamma}\lambda (E-1) \kappa + \frac{\gamma}{1-\gamma}\sqrt{\lambda \log \frac{|\calS||\calA|KT}{\delta}}, 
\end{align} 
for all $(s, a)\in \calS\times \calA, i\in [T], k\in [K]$. From Eq.\,\eqref{eq: error: within a round: term2: decomposed} to Eq.\,\eqref{eq: error: within a round: term2: hoeff}, {since for each timestep, each agent independently samples the next state for all state-action pairs to form the sample transition matrix, we can use Hoeffding's inequality by treating $\lambda(1-\lambda)^{i-j-1}({P}_j^k - \tilde{P}_j^k)V^*$ as the independent random variables with their absolute values bounded by $\lambda(1-\lambda)^{i-j-1}\linf{V^*}$.}}

In addition, we have 
\begin{align}
\label{eq: error: within a round: term 3}
{\linf{\gamma \lambda \sum_{j=\chi(i)}^{i-1} (1-\lambda)^{i-j-1} \tilde{P}_j^k (V^* - V_j^k)}}
\le \gamma \lambda \sum_{j=\chi(i)}^{i-1} (1-\lambda)^{i-j-1} \linf{\Delta_j^k}. 
\end{align} 
Combining the bounds in Eq.\,\eqref{eq: error: within a round: term 1}, Eq.\,\eqref{eq: error: within a round: term 2}, and Eq.\,\eqref{eq: error: within a round: term 3}, we get 
\begin{align}
\linf{\Delta_{i}^k}
\label{eq: coarse bound on local error}
\nonumber 
&\le  (1-\lambda)^{i-\chi(i)} \linf{\Delta_{\chi(i)}} 
+ \frac{\gamma}{1-\gamma}\lambda (E-1) \kappa + \frac{\gamma}{1-\gamma}\sqrt{\lambda \log \frac{|\calS||\calA|KT}{\delta}} \\
\nonumber
& \quad + \gamma\lambda \sum_{j=\chi(i)}^{i-1} (1-\lambda)^{i-j-1} \linf{\Delta_j^k}\\
\nonumber 
& \le {\pth{1-(1-\gamma)\lambda}^{i-\chi(i)}\linf{\Delta_{\chi(i)}}} \\
& \quad + \pth{1+ \gamma \lambda}^{i-\chi(i)}\pth{\frac{\gamma}{1-\gamma}\lambda (E-1)\kappa + \frac{\gamma}{1-\gamma}\sqrt{\lambda \log \frac{|\calS||\calA|KT}{\delta}}}, 
\end{align}
where the last inequality can be shown via inducting on $i-\chi(i) \in \{0, \cdots, E-1\}$. 
When $\lambda\le \frac{1}{E}$,  
\begin{align*}
\pth{1+ \gamma \lambda}^{i-\chi(i)}
\le \pth{1+ \lambda}^{E} \le \pth{1+ 1/E}^{E} \le e \le 3. 
\end{align*}
We get 
\begin{align*}
\linf{\Delta_{i}^k}
\le \linf{\Delta_{\chi(i)}} + 3\frac{\gamma}{1-\gamma}\lambda (E-1)\kappa + 3\frac{\gamma}{1-\gamma}\sqrt{\lambda \log \frac{|\calS||\calA|KT}{\delta}}. 
\end{align*}

\section{Proof of Theorem \ref{thm: sample complexity: synchronous}}
%
%
By Lemma \ref{lm: error iteration}, 
    \begin{align*}
       \Delta_{t+1}&=(1-\lambda)^{t+1}\Delta_0 + \sum_{i=0}^t(1-\lambda)^i\frac{\gamma\lambda}{K}\sum_{k=1}^K (\bar P-\tilde P_{t-i}^k)V^*+\sum_{i=0}^t(1-\lambda)^i\frac{\gamma\lambda}{K}\sum_{k=1}^K\tilde P_{t-i}^k(V^*-V_{t-i}^k). 
    \end{align*}

Taking the $\ell_{\infty}$ norm on both sides, we get 
\begin{align}
\label{eq: error}
\nonumber
\linf{\Delta_{t+1}}  
&\leq 
\underbrace{(1-\lambda)^{t+1}\linf{\Delta_0}}_{\textnormal{(I.1)}}
+ \underbrace{\linf{\sum_{i=0}^t(1-\lambda)^i\lambda \gamma\frac{1}{K}\sum_{k=1}^K (\bar P-\tilde P_{t-i}^k)V^*}}_{\textnormal{(I.2)}} \\
& \quad +\underbrace{\sum_{i=0}^t(1-\lambda)^i\lambda \gamma \linf{\frac{1}{K}\sum_{k=1}^K\tilde P_{t-i}^k(V^*-V_{t-i}^k)}}_{\textnormal{(I.3)}}. 
\end{align}
We bound the three terms in the right-hand-side of the above-displayed equation separately. 

\noindent {\bf Bounding (I.1).}
Since $0\le Q_0(s,a)\le \frac{1}{1-\gamma}$, the first term can be bounded as 
\begin{align}
\label{eq: bound: term 1}
\textnormal{(I.1)} = (1-\lambda)^{t+1}\linf{\Delta_0} \le  (1-\lambda)^{t+1}\frac{1}{1-\gamma}. 
\end{align}

\noindent {\bf Bounding (I.2).}
To bound the second term \textnormal{(I.2)} in Eq.\,\eqref{eq: error}, 
we have 
\begin{align*}
\sum_{i=0}^t(1-\lambda)^i\lambda \gamma\frac{1}{K}\sum_{k=1}^K (\bar P-\tilde P_{t-i}^k)V^* 
& = \sum_{i=0}^t(1-\lambda)^i\lambda \gamma\frac{1}{K}\sum_{k=1}^K (P^k-\tilde P_{t-i}^k)V^*\\
& = \frac{1}{K}\sum_{k=1}^K \sum_{i=0}^t(1-\lambda)^i\lambda \gamma (P^k-\tilde P_{t-i}^k)V^*. 
\end{align*}
Let {$X_{i,k}=\frac{1}{K}\gamma \lambda (1-\lambda)^i (P^k - \tilde{P}_{t-i}^k)V^*$}. It is easy to see that $\expect{X_{i,k}(s,a)} = 0$ for all $(s,a)$. 
By Lemma \ref{lm: coarse bound}, we have ${\abth{X_{i,k}(s,a)}} \le \frac{2}{K(1-\gamma)}\gamma \lambda (1-\lambda)^i$ for all $(s,a)$. 
Since the sampling across clients and across iterations are independent, via invoking Hoeffding's inequality, for any given $\delta\in (0,1)$, with probability at least $1-\delta$, 
\begin{align}
\label{eq: bound: sample: time-invariant: term 2}
\textnormal{(I.2)} 
= \linf{\sum_{i=0}^t(1-\lambda)^i\lambda \gamma\frac{1}{K}\sum_{k=1}^K (\bar P-\tilde P_{t-i}^k)V^*}
&\le  \frac{\gamma}{1-\gamma} \sqrt{\frac{1}{K}\lambda \log \frac{|\calS||\calA|TK}{\delta}}. 
\end{align}

\noindent {\bf Bounding (I.3).}
To bound the third term \textnormal{(I.3)} in Eq.\,\eqref{eq: error}, 
following the roadmap of \citet{woo_blessing_2023}, we divide the summation into two parts as follows.  
For any $\beta E \le t \le T$, 
we have  
\begin{align*}
&\textnormal{(I.3)} = \sum_{i=0}^t(1-\lambda)^i\lambda \gamma \linf{\frac{1}{K}\sum_{k=1}^K\tilde P_{t-i}^k(V^*-V_{t-i}^k)} \\
&=\sum_{i=0}^t(1-\lambda)^{t-i}\lambda \gamma \linf{\frac{1}{K}\sum_{k=1}^K\tilde P_{i}^k(V^*-V_{i}^k)}\\
& = \sum_{i=0}^{\chi(t)-\beta E} (1-\lambda)^{t-i}\lambda \gamma \linf{\frac{1}{K}\sum_{k=1}^K\tilde P_{i}^k(V^*-V_{i}^k)} 
+ \sum_{i= \chi(t)-\beta E+1}^{t}(1-\lambda)^{t-i}\lambda \gamma \linf{\frac{1}{K}\sum_{k=1}^K\tilde P_{i}^k(V^*-V_{i}^k)} \\
& \le \frac{\gamma}{1-\gamma} (1-\lambda)^{t-\chi(t)+\beta E}
+ \sum_{i= \chi(t)-\beta E+1}^{t}(1-\lambda)^{t-i}\lambda \gamma \linf{\frac{1}{K}\sum_{k=1}^K\tilde P_{i}^k(V^*-V_{i}^k)}. 
\end{align*}
By Lemma \ref{lm: sample complexity: stepsize push}, 
\begin{align*}
&\sum_{i= \chi(t)-\beta E+1}^{t}(1-\lambda)^{t-i}\lambda \gamma \linf{\frac{1}{K}\sum_{k=1}^K\tilde P_{i}^k(V^*-V_{i}^k)} \\ 
&\le \sum_{i= \chi(t)-\beta E+1}^{t}(1-\lambda)^{t-i}\lambda \gamma  
\left(\linf{\Delta_{\chi(i)}} + 2\lambda \frac{1}{K}\sum_{k=1}^K\sum_{j = \chi(i)}^{i-1} \linf{\Delta_{t^{\prime}}^k} \right. \\
& \quad \quad \quad \quad \quad \quad \left.+  \gamma \lambda \frac{1}{K}\sum_{k=1}^K \max_{s, a} \abth{\sum_{j = \chi(i)}^{i-1} \pth{\tilde{P}_{j}^k(s, a) - \bar{P}(s, a) }V^*}\right).  
\end{align*}
Since $\tilde{P}_j^k(s,a)$'s are independent across time $j$ and across state action pair $(s,a)$, and $|\tilde{P}_{j}^k(s, a) - \bar{P}(s, a) V^*| \le \frac{1}{1-\gamma}$ (from Lemma \ref{lm: coarse bound}), with Hoeffding's inequality and union bound, we get for any $\delta\in (0,1)$, with probability at least $1-\delta$, 
\begin{align}
\label{eq: perturbation between synchronization without decay}
\abth{\sum_{j = \chi(i)}^{i-1} \pth{\tilde{P}_{j}^k(s, a) - \bar{P}(s, a) }V^*} \le (E-1)\frac{1}{1-\gamma}\kappa +\frac{1}{1-\gamma} \sqrt{(E-1) \log \frac{|\calS|\calA|KT}{\delta}}   
\end{align}
for all $(s,a)\in \calS\times \calA$, $k\in {K}$, and $i$. 
By Lemma \ref{lm: sample complexity: stepsize push: alternative}, with probability at least $(1-\delta)$, we have   
\begin{align*}
&\sum_{i= \chi(t)-\beta E+1}^{t}(1-\lambda)^{t-i}\lambda \gamma  2\lambda \frac{1}{K}\sum_{k=1}^K\sum_{j = \chi(i)}^{i-1} \linf{\Delta_{j}^k} \\ 
& \le 2\lambda^2 \gamma \sum_{i= \chi(t)-\beta E+1}^{t}(1-\lambda)^{t-i} \frac{1}{K}\sum_{k=1}^K \sum_{j = \chi(i)}^{i-1}
\pth{\linf{\Delta_{\chi(i)}} + 3\frac{\gamma}{1-\gamma}\lambda (E-1) \kappa + 3\frac{\gamma}{1-\gamma}\sqrt{\lambda \log \frac{|\calS||\calA|KT}{\delta}}} \\
& \le 2\lambda \gamma (E-1) \max_{\chi(t)-\beta E\le i\le t} \linf{\Delta_{\chi(i)}} 
+ \frac{6\gamma^2\lambda^2}{1-\gamma}(E-1)^2\kappa 
+ \frac{6\gamma^2 \lambda}{1-\gamma}(E-1)\sqrt{\lambda \log \frac{|\calS||\calA|KT}{\delta}}. 
\end{align*}
Thus, {by applying the union bound}, we get with probability at least $(1-2\delta)$, 
\begin{align*}
&\sum_{i= \chi(t)-\beta E+1}^{t}(1-\lambda)^{t-i}\lambda \gamma \linf{\frac{1}{K}\sum_{k=1}^K\tilde P_{i}^k(V^*-V_{i}^k)} \\ 
& \le \gamma \max_{\chi(t)-\beta E\le i\le t} \linf{\Delta_{\chi(i)}}  +
2\lambda \gamma (E-1) \max_{\chi(t)-\beta E\le i\le t} \linf{\Delta_{\chi(i)}} 
+ \frac{6\gamma^2\lambda^2}{1-\gamma}(E-1)^2 \kappa \\
& \quad + \frac{6\gamma^2 \lambda}{1-\gamma}(E-1)\sqrt{\lambda \log \frac{|\calS||\calA|KT}{\delta}} \\
& \quad +  
\sum_{i= \chi(t)-\beta E+1}^{t}(1-\lambda)^{t-i}\lambda \gamma \pth{\frac{\gamma \lambda}{1-\gamma}(E-1) \kappa + \frac{\gamma \lambda}{1-\gamma }\sqrt{(E-1)\log \frac{|\calS||\calA|KT}{\delta}}} \\
& = \gamma(1+2\lambda(E-1)) \max_{\chi(t)-\beta E\le i\le t} \linf{\Delta_{\chi(i)}} + \frac{\gamma^2}{1-\gamma} (6\lambda^2(E-1)^2+\lambda(E-1)) \kappa\\
& \quad +\frac{\gamma^2 \lambda}{1-\gamma}\sqrt{(E-1)\log \frac{|\calS||\calA|KT}{\delta}} 
+ \frac{6\gamma^2 \lambda}{1-\gamma}(E-1)\sqrt{\lambda \log \frac{|\calS||\calA|KT}{\delta}}. 
\end{align*}

Hence, the third term \textnormal{(I.3)} in Eq.\,\eqref{eq: error} can be bounded as  
\begin{align}
\label{eq: bound: term 3}
\nonumber 
&\textnormal{(I.3)}  = \sum_{i=0}^t(1-\lambda)^i\lambda\gamma\linf{ \frac{1}{K}\sum_{k=1}^K\tilde P_{i}^k(V^*-V_{i}^k)} \\
\nonumber 
& \le \frac{\gamma}{1-\gamma} (1-\lambda)^{t-\chi(t)+\beta E} 
+ \gamma(1+2\lambda(E-1)) \max_{\chi(t)-\beta E\le i\le t} \linf{\Delta_{\chi(i)}} + \frac{\gamma^2}{1-\gamma} (6\lambda^2(E-1)^2+\lambda(E-1)) \kappa\\
& +\frac{\gamma^2 \lambda}{1-\gamma}\sqrt{(E-1)\log \frac{|\calS||\calA|KT}{\delta}} 
+ \frac{6\gamma^2 \lambda}{1-\gamma}(E-1)\sqrt{\lambda \log \frac{|\calS||\calA|KT}{\delta}}. 
\end{align}
%

\noindent {\bf Combing the bounds of (I.1), (I.2), and (I.3) in Eq.\,\eqref{eq: error}.}
Combing the bounds for terms (\ref{eq: bound: term 1}), (\ref{eq: bound: sample: time-invariant: term 2}), and (\ref{eq: bound: term 3}), we get the following recursion holds for all rounds $T$ with probability at least $(1-3\delta)$:  
\begin{align*}
\linf{\Delta_{t+1}} 
& \le 
 (1-\lambda)^{t+1}\frac{1}{1-\gamma}  
+ \frac{\gamma}{1-\gamma}\sqrt{\frac{1}{K}\lambda \log\frac{|\calS||\calA|TK}{\delta}} +\frac{\gamma}{1-\gamma} (1-\lambda)^{t-\chi(t)+\beta E}\\
&\quad + \gamma(1+2\lambda(E-1)) \max_{\chi(t)-\beta E\le i\le t} \linf{\Delta_{\chi(i)}} + \frac{\gamma^2}{1-\gamma} (6\lambda^2(E-1)^2+\lambda(E-1)) \kappa\\
& \quad + \frac{\gamma^2 \lambda}{1-\gamma}\sqrt{(E-1)\log \frac{|\calS||\calA|KT}{\delta}} 
+ \frac{6\gamma^2 \lambda}{1-\gamma}(E-1)\sqrt{\lambda \log \frac{|\calS||\calA|KT}{\delta}}\\
& \le \gamma(1+2\lambda(E-1)) \max_{\chi(t)-\beta E\le i\le t} \linf{\Delta_{\chi(i)}} + \frac{2}{1-\gamma} (1-\lambda)^{\beta E} + \frac{\gamma^2}{1-\gamma} (6\lambda^2(E-1)^2+\lambda(E-1))  \kappa \\
&\quad + \frac{\gamma^2 \lambda}{1-\gamma}\sqrt{(E-1)\log \frac{|\calS||\calA|KT}{\delta}} 
+ \frac{6\gamma^2 \lambda}{1-\gamma}(E-1)\sqrt{\lambda \log \frac{|\calS||\calA|KT}{\delta}} \\
& \quad + \frac{\gamma}{1-\gamma}\sqrt{\frac{1}{K}\lambda \log\frac{|\calS||\calA|TK}{\delta}}. 
\end{align*}
%
%
%
Let 
\begin{align}
\label{eq: rho}
\nonumber 
\rho := &\frac{2}{1-\gamma} (1-\lambda)^{\beta E} + \frac{\gamma^2}{1-\gamma} (6\lambda^2(E-1)^2+\lambda(E-1))  \kappa  \\
& \nonumber + \frac{\gamma^2 \lambda}{1-\gamma}\sqrt{(E-1)\log \frac{|\calS||\calA|KT}{\delta}} 
+ \frac{6\gamma^2 \lambda}{1-\gamma}(E-1)\sqrt{\lambda \log \frac{|\calS||\calA|KT}{\delta}} \\
& + \frac{\gamma}{1-\gamma}\sqrt{\frac{1}{K}\lambda \log\frac{|\calS||\calA|TK}{\delta}}.
\end{align}
With the assumption that \(\lambda \leq \frac{1-\gamma}{4\gamma(E-1)}\), the above recursion can be written as \[\linf{\Delta_{t+1}} \le \frac{1+\gamma}{2}\max_{\chi(t)-\beta E\le i\le t} \linf{\Delta_{\chi(i)}}+ \rho. \]
Unrolling the above recursion $L$ times where $L\beta E \le t \le T$, we obtain that 
\begin{align*}
\linf{\Delta_{t+1}} &\le (\frac{1+\gamma}{2})^L\max_{\chi(t)-L\beta E\le i\le t} \linf{\Delta_{\chi(i)}}+ \sum_{i=0}^{L-1}(\frac{1+\gamma}{2})^i \rho\\
&\leq (\frac{1+\gamma}{2})^L\frac{1}{1-\gamma} + \frac{2}{1-\gamma}\rho. 
\end{align*}
{Choosing $ \beta =\lf{\frac{1}{E}\sqrt{\frac{(1-\gamma)T}{2\lambda}}}\rf$, $L=\lc{\sqrt{\frac{\lambda T}{1-\gamma}}}\rc$}, $t+1 = T$, we get 
\begin{align*}
\linf{\Delta_{T}} 
& \le \frac{1}{1-\gamma} (\frac{1+\gamma}{2})^{\sqrt{\frac{\lambda T}{1-\gamma}}}
+ \frac{2}{1-\gamma}\left(\frac{2}{1-\gamma} (1-\lambda)^{\beta E} + \frac{\gamma^2 }{1-\gamma} (6\lambda^2(E-1)^2+\lambda(E-1))  \kappa \right. \\
& \nonumber \quad + \left.\pth{\frac{6\gamma^2 \lambda}{1-\gamma}\sqrt{E-1}+\frac{\gamma^2 \sqrt{\lambda}}{1-\gamma}}\sqrt{\lambda (E-1) \log \frac{|\calS||\calA|KT}{\delta}} 
+\frac{\gamma}{1-\gamma}\sqrt{\frac{1}{K}\lambda \log\frac{|\calS||\calA|TK}{\delta}} \right)\\
& \leq \frac{1}{1-\gamma} \exp\sth{-\frac{1}{2}\sqrt{(1-\gamma)\lambda T}}
+ \frac{4}{(1-\gamma)^2}\exp\sth{-\frac{1}{2}\sqrt{(1-\gamma)\lambda T}}\\
&\quad+ \frac{2\gamma^2}{(1-\gamma)^2} (6\lambda^2(E-1)^2+\lambda(E-1)) \kappa \\
& \nonumber \quad +\pth{\frac{12\gamma^2 \lambda}{(1-\gamma)^2}\sqrt{E-1}+\frac{2\gamma^2 \sqrt{\lambda}}{(1-\gamma)^2}}\sqrt{\lambda (E-1) \log \frac{|\calS||\calA|KT}{\delta}}  
+\frac{2\gamma}{(1-\gamma)^2}\sqrt{\frac{1}{K}\lambda \log\frac{|\calS||\calA|TK}{\delta}}\\
&\leq \frac{4}{(1-\gamma)^2} \exp\sth{-\frac{1}{2}\sqrt{(1-\gamma)\lambda T}}
+ \frac{2\gamma^2}{(1-\gamma)^2} (6\lambda^2(E-1)^2+\lambda(E-1)) \kappa \\
\nonumber 
&\quad {+\pth{\frac{14\gamma^2 \lambda}{(1-\gamma)^2}\sqrt{E-1}}\sqrt{\log \frac{|\calS||\calA|KT}{\delta}}} 
+\frac{2\gamma}{(1-\gamma)^2}\sqrt{\frac{1}{K}\lambda \log\frac{|\calS||\calA|TK}{\delta}}, 
\end{align*}  
where the second inequality follows from 
\begin{align*}
    &(\frac{1+\gamma}{2})^{\sqrt{\frac{\lambda T}{1-\gamma}}}=(1-\frac{1-\gamma}{2})^{\sqrt{\frac{\lambda T}{1-\gamma}}}  \leq \exp\sth{-\frac{1}{2}\sqrt{(1-\gamma)\lambda T}},\\
    &(1-\lambda)^{\beta E} \leq \exp\sth{-\lambda \sqrt{\frac{(1-\gamma)T}{2\lambda}}} \leq \exp\sth{-\frac{1}{2}\sqrt{(1-\gamma)\lambda T}}.
\end{align*}
{By the assumption that $(E-1)\leq \frac{1}{K\lambda}$, the above can be further simplified as 
\begin{align*}
\linf{\Delta_{T}} 
&\leq \frac{4}{(1-\gamma)^2} \exp\sth{-\frac{1}{2}\sqrt{(1-\gamma)\lambda T}} + \frac{14\gamma^2}{(1-\gamma)^2} \lambda(E-1) \kappa +\frac{16}{(1-\gamma)^2}\sqrt{\frac{\lambda}{K} \log \frac{|\calS||\calA|KT}{\delta}}. 
\end{align*}}

\newpage
\section{Proof of \prettyref{thm:impossibility}}
\label{app: lower bound}


Let $|\calA|=1$, in which case $Q$-function coincides with the $V$-function. 
According to Algorithm~\ref{algorithm:pfedda}, when $(t+1)\text{ mod } E \ne 0$, we have 
\[
Q_{t+1}^k = \pth{(1-\lambda)I + \lambda \gamma P^k} Q_t^k + \lambda R. 
\]
Define $A^k \triangleq (1-\lambda)I + \lambda \gamma P^k$. 
We obtain the following recursion between two synchronization rounds:
\[
Q_{(z+1)E}^k = (A^k)^E Q_{z E}^k + \pth{(A^k)^0+\dots (A^k)^{E-1}}\lambda R. 
\]
Define 
\begin{equation}
    \label{eq:bar-A}
    \bar{A}^{(\ell)} \triangleq \frac{1}{K}\sum_{k=1}^K (A^k)^\ell.
\end{equation}
Note that $Q^*$ is the fixed point under the transition kernel $\bar{P}$, we have $\lambda R = \lambda (I-\gamma \bar{P})Q^* = (I-\bar{A}^{(1)})Q^*$ since $\bar A^{(1)} = I-\lambda(I-\gamma\bar P)$.
Furthermore, since $Q_{t E}^1,\dots,Q_{t E}^K$ are identical due to synchronization, we get
\[
\bar{Q}_{(z+1)E}
=\bar{A}^{(E)} \bar{Q}_{zE} + \pth{I+\bar{A}^{(1)}+\dots \bar{A}^{(E-1)}} \pth{I-\bar{A}^{(1)}}Q^*. 
\]
Consequently,
\begin{align}
\Delta_{(z+1)E} 
&= Q^*-\bar Q_{(z+1)E}\nonumber\\
&= \bar{A}^{(E)} \Delta_{zE} + \pth{ \pth{I- \bar{A}^{(E)} } -   \pth{I+\bar{A}^{(1)}+\dots \bar{A}^{(E-1)}} \pth{I-\bar{A}^{(1)}} } Q^*. \label{eq:Delta_t-recursion}
\end{align}


Next, consider $|\calS|=2$ and even $K$ with 
\[
P^{2k-1}=\begin{bmatrix}
    1& 0\\
    0& 1
\end{bmatrix},\quad 
P^{2k}=\begin{bmatrix}
    0& 1\\
    1& 0
\end{bmatrix},\quad \text{for}~k\in\naturals. 
\]
Then $\bar P = \frac{1}{2}\mathbf{11}^\top$, where $\mathbf{1}$ denotes the all ones vector.
For the above transition kernels, we have 
\[
\frac{1}{K}\sum_{k=1}^K (P^{k})^\ell =\begin{cases}
    I,& \ell~\text{even},\\
    \bar P,& \ell~\text{odd}.\\
\end{cases}
\]
Applying the definition of $\bar{A}^{(\ell)}$ in~\eqref{eq:bar-A} yields that 
\begin{align*}
    \bar{A}^{(\ell)}&= \frac{1}{K}\sum_{k=1}^K (A^k)^\ell\\
    &=\frac{1}{K}\sum_{k=1}^K((1-\lambda)I + \lambda \gamma P^k)^\ell\\
    &=\frac{1}{K}\sum_{k=1}^K \sum_{j=0}^\ell\binom{\ell}{j} (\lambda \gamma P^k)^{j} ((1-\lambda)I)^{\ell-j}\\
    &= \sum_{j~\text{even}}\binom{\ell}{j}(1-\lambda)^{\ell-j} (\lambda \gamma)^j (I-\bar P + \bar P)
    + \sum_{j~\text{odd}}\binom{\ell}{j}(1-\lambda)^{\ell-j} (\lambda \gamma)^j \bar P\\
    &= \underbrace{
    \frac{1}{2}((1-\lambda - \lambda \gamma)^\ell + (1-\lambda + \lambda \gamma)^\ell)
    }_{\triangleq \alpha_\ell} (I-\bar P)
    + \underbrace{(1-\lambda + \lambda \gamma)^\ell}_{\triangleq \beta_\ell} \bar P\\
    & = \alpha_\ell (I-\bar P) + \beta_\ell \bar P,
\end{align*}
which is the eigen-decomposition of $\bar{A}^{(\ell)}$.
Let 
\[
\lambda_1\triangleq(1+\gamma)\lambda, \lambda_2\triangleq(1-\gamma)\lambda,
\quad \nu_1=1-\lambda_1, \nu_2=1-\lambda_2.
\]
Then 
\begin{equation}
    \label{eq:alpha-beta}
    \alpha_\ell = \frac{1}{2}(\nu_1^\ell + \nu_2^\ell),
    \quad 
    \beta_\ell=\nu_2^\ell.
\end{equation}

Note that $0\le \alpha\le \beta\le 1$ and $I-\bar P$ and $\bar P$ are orthogonal projection matrices satisfying $(I-\bar P)\bar P = 0$. 
The matrices for the second term of the error on the right-hand side of~\ref{eq:Delta_t-recursion} reduce to
\begin{align*}
&\phantom{{}={}} \pth{I+\bar{A}^{(1)}+\dots \bar{A}^{(E-1)}} \pth{I-\bar{A}^{(1)}}\\
& = \pth{\sum_{\ell=0}^{E-1}\alpha_\ell(I-\bar P) +\sum_{\ell=0}^{E-1}\beta_\ell \bar P} \pth{(\alpha_0-\alpha_1)(I-\bar P)+(\beta_0-\beta_1)\bar P}\\
& = \pth{(1-\alpha_1)\sum_{\ell=0}^{E-1}\alpha_\ell(I-\bar P)^2 +(1-\beta_1)\sum_{\ell=0}^{E-1}\beta_\ell \bar P^2} \text{ since } \alpha_0=\beta_0=1 \\
& = \pth{(1-\alpha_1)\sum_{\ell=0}^{E-1}\alpha_\ell(I-\bar P) +(1-\beta_1)\sum_{\ell=0}^{E-1}\beta_\ell \bar P} \text{ since } (I-\bar P) \text{ and } \bar P \text{ are idempotent}.
\end{align*}
It follow that
\begin{align*}
&\phantom{{}={}}\pth{I- \bar{A}^{(E)} } -   \pth{I+\bar{A}^{(1)}+\dots \bar{A}^{(E-1)}} \pth{I-\bar{A}^{(1)}} \\
& = \underbrace{\pth{ (1-\alpha_E)-(1-\alpha_1)\pth{\sum_{i=0}^{E-1}\alpha_i} }}_{\triangleq \kappa_E} (I-\bar P) 
+ \underbrace{\pth{ (1-\beta_E)-(1-\beta_1)\pth{\sum_{i=0}^{E-1}\beta_i}  }}_{=0} \bar P.
\end{align*}
Applying \eqref{eq:alpha-beta} yields that
\begin{equation}
\label{eq:kappa}
\kappa_E= - \frac{\gamma}{2}\pth{\frac{1-\nu_2^E}{1-\gamma}- \frac{1-\nu_1^E}{1+\gamma}}.
\end{equation}

It follows from~\eqref{eq:Delta_t-recursion} that the error evolves as
\[
\Delta_{(z+1)E}
= \pth{\alpha_E (I-\bar P) + \beta_E \bar P} \Delta_{zE}
+ \kappa_E (I-\bar P) Q^*,
\]
which further yields the following full recursion of the error:
\begin{align*}
\Delta_{zE}
& = \pth{\alpha_E (I-\bar P) + \beta_E \bar P}^z \Delta_{0} + \sum_{\ell = 0}^{z-1}\pth{\alpha_E (I-\bar P) + \beta_E \bar P}^\ell \kappa_E (I-\bar P) Q^* \\
& = \pth{\alpha_E^z (I-\bar P) + \beta_E^z \bar P} \Delta_{0} + \sum_{\ell = 0}^{z-1}\pth{\alpha_E^\ell (I-\bar P) + \beta_E^\ell \bar P}  \kappa_E (I-\bar P) Q^* \\ &\phantom{{}={}}\text{ since } \pth{\alpha_E (I-\bar P) + \beta_E \bar P}^\ell = \alpha_E^\ell (I-\bar P) + \beta_E^\ell \bar P, \forall \ell \in \mathbb{N}\\
& = \pth{\alpha_E^z (I-\bar P) + \beta_E^z \bar P} \Delta_{0} + \frac{1-\alpha_E^z}{1-\alpha_E}\kappa_E (I-\bar P) Q^* \\
& = \pth{\alpha_E^z + \frac{1-\alpha_E^z}{1-\alpha_E}\kappa_E  } (I-\bar P)Q^* + \beta_E^z \bar P Q^*,
\end{align*}
where the last equality applied the zero initialization condition. 


Note that $(I-\bar P)Q^*$ and $\bar P Q^*$ are orthogonal vectors. 
Since $|\calS|=2$, we have 
\begin{align*}
    &\linf{\Delta_{zE}}
\ge \frac{1}{\sqrt{2}} \norm{\Delta_{zE}}
\ge \frac{\min\{\|(I-\bar P)Q^*\|_2, \|\bar P Q^*\|_2 \}}{\sqrt{2}} \cdot \max\sth{|\alpha_E^z + \frac{1-\alpha_E^z}{1-\alpha_E}\kappa_E|, \beta_E^z}. 
\end{align*}
{Let $R = \begin{bmatrix}
    r_1\\r_2
\end{bmatrix}$, since $Q^*=(I-\gamma \bar P)^{-1}R = (I-\bar P)R + \frac{1}{1-\gamma}\bar P R$, we obtain that
\[
(I-\bar P)Q^*
= (I-\bar P) R = \frac{1}{2}\begin{bmatrix}
    r_1-r_2 \\ r_2-r_1
\end{bmatrix},
\qquad
\bar P Q^*
= \frac{1}{1-\gamma}\bar P R = \frac{1}{2(1-\gamma)}\begin{bmatrix}
    r_1+r_2 \\ r_1+r_2
\end{bmatrix}.
\] 
\[
\|(I-\bar P)Q^*\|_2
= \frac{\sqrt{2}}{2}\abth{r_1-r_2},
\qquad
\|\bar P Q^*\|_2
= \frac{\sqrt{2}}{2(1-\gamma)}\abth{r_1+r_2}.
\] 

When $r_1=r_2,$ the error $\Delta_{zE}$ reduces to $\beta_E^z \bar P Q^*$, and $\linf{\Delta_{zE}} = \frac{1}{2(1-\gamma)}\abth{r_1+r_2}\abth{\beta_E^z}  $;
otherwise, $\min\{\|(I-\bar P)Q^*\|_2, \|\bar P Q^*\|_2 \}=\frac{\sqrt{2}}{2}\min\{\abth{r_1-r_2}, \frac{1}{1-\gamma}\abth{r_1+r_2}\}.$ }
It remains to analyze the coefficients as functions of $\lambda$.
To this end, we introduce the following lemma:

\begin{lemma}
\label{lmm:kappa_E}
The following properties hold:
\begin{enumerate}
    \item Negativity: $\kappa_E<0$;
    \item Monotonicity: $\frac{\kappa_E}{1-\alpha_E}$ is monotonically decreasing for $\lambda \in (0,\frac{1}{1+\gamma})$;
    \item Upper bound: $|\frac{\kappa_E}{1-\alpha_E}|\le \frac{\gamma^2}{1-\gamma^2}$ for $\lambda\in (0,\frac{1}{1+\gamma})$;
    \item Lower bound: if $(1+\gamma)\lambda \le \frac{1}{2E}$, then $|\frac{\kappa_E}{1-\alpha_E}|\ge \frac{\lambda \gamma^2(E-1)}{4}$.
\end{enumerate}
\end{lemma}
\begin{proof}
We prove the properties separately. 
\begin{enumerate}
\item Note that $\nu_1< \nu_2$, $1-\nu_1=(1+\gamma)\lambda$, and $1-\nu_2=(1-\gamma)\lambda$.
    Then it follows from \eqref{eq:kappa} that
    \[
    \kappa_E = -\frac{\lambda \gamma}{2}\sum_{i=1}^{E-1} (\nu_2^i - \nu_1^i) < 0.
    \]
\item For the monotonicity, it suffices to show that $\frac{d}{d\lambda} \frac{\kappa_E}{1-\alpha_E}\le 0$.
    We calculate the derivative as
    \[
    \frac{d}{d\lambda} \frac{\kappa_E}{1-\alpha_E}
    = \frac{\gamma E (1-\nu_1^E)(1-\nu_2^E)}{2(1-\gamma^2)(1-\alpha_E)^2}
    \pth{ \frac{(1+\gamma)\nu_1^{E-1}}{1- \nu_1^E} -  \frac{(1-\gamma)\nu_2^{E-1}}{1- \nu_2^E} }.
    \]
    Note that
    \[
    \frac{(1+\gamma)\nu_1^{E-1}}{1- \nu_1^E} -  \frac{(1-\gamma)\nu_2^{E-1}}{1- \nu_2^E}
    = \frac{1}{\lambda}\pth{\frac{\nu_1^{E-1}}{1+\nu_1+\dots+\nu_1^{E-1}} - \frac{\nu_2^{E-1}}{1+\nu_2+\dots+\nu_2^{E-1}}}\le 0.
    \]
\item For the upper bound, it suffices to show the result at $\lambda=\frac{1}{1+\gamma}$ due to the negativity and monotonicity. 
    At $\lambda=\frac{1}{1+\gamma}$, we have
    \[
    \abth{\frac{\kappa_E}{1-\alpha_E}}
    =\frac{\gamma}{1-\gamma^2}\pth{\gamma - \frac{(\frac{2\gamma}{1+\gamma})^E}{2-(\frac{2\gamma}{1+\gamma})^E}}
    \le \frac{\gamma^2}{1-\gamma^2}.
    \]
\item     For the lower bound, the case $E=1$ trivially holds.
    Next, consider $E\ge 2$. We have
    \begin{align*}
    \frac{\kappa_E}{1-\alpha_E}
    & = -\frac{\gamma}{1-\gamma^2} \frac{(1+\gamma)(1-\nu_2^E) - (1-\gamma)(1-\nu_1^E) }{(1-\nu_1^E)+(1-\nu_2^E)} \\
    & = -\lambda \gamma \frac{\sum_{\ell=1}^{E-1}(\nu_2^{\ell}- \nu_1^{\ell})}{(1-\nu_1^E)+(1-\nu_2^E)}.
    \end{align*}
    Note that $1-nx \le (1-x)^n\le 1-\frac{1}{2}nx$ for $n\ge 1$ and $0\le x\le \frac{1}{n}$.
    Then, for $(1+\gamma)\lambda \le \frac{1}{2E}$, we have
    \begin{align*}
        \nu_1^E & = (1-(1+\gamma)\lambda)^E \ge 1-(1+\gamma)\lambda E \ge \frac{1}{2},\\
        \nu_2^E & = (1-(1-\gamma)\lambda)^E \ge 1-(1-\gamma)\lambda E.
    \end{align*}
    Moreover, for all $x\in [\nu_1,\nu_2]\subseteq[0,1]$ and $\ell-1\le E$, we have
    \[
    x^{\ell-1}\ge x^E\ge \nu_1^E\ge \frac{1}{2}.
    \]
    We obtain that
    \[
    \frac{\sum_{\ell=1}^{E-1}(\nu_2^{\ell}- \nu_1^{\ell})}{(1-\nu_1^E)+(1-\nu_2^E)}
    \ge \frac{ \sum_{\ell=1}^{E-1} \int_{\nu_1}^{\nu_2} \ell \cdot x^{\ell-1}dx }{2\lambda E}
    \ge \frac{ \sum_{\ell=1}^{E-1}  \ell \frac{1}{2} (\nu_2-\nu_1) }{2\lambda E}
    = \frac{1}{4}\gamma (E-1).
    \]    
\end{enumerate} 
    The proof is completed.
\end{proof}

We consider two regimes of the stepsize separated by $\lambda_0 \triangleq \frac{\log z}{(1-\gamma) z E}< \frac{1}{1+\gamma}$, where the dominating error is due to the small stepsize and the environment heterogeneity, respectively:


\paragraph{Slow rate due to small stepsize when $\lambda \le \lambda_0 $.} 
Since $\beta_E^z$ monotonically decreases as $\lambda$ increases,
\[
\beta_E^z = (1-(1-\gamma)\lambda)^{zE}\geq  (1-(1-\gamma)\lambda_0)^{zE} = \pth{1-\frac{\log z}{zE} }^{zE}.
\]
Note that $\frac{\log z }{z E} \in (0,\frac{1}{2})$, applying the fact $\log(1-x)+x \ge -x^2$ for $x\in[0,\frac{1}{2}]$ yields that
\[
\log\pth{1-\frac{\log z}{zE} } + \frac{\log r}{zE} 
\ge -\pth{\frac{\log z}{zE}}^2 \ge -\frac{1}{z E}.
\]
Then we get
\[
\beta_E^z
\ge \pth{1-\frac{\log z}{zE} }^{zE} 
\ge \frac{1}{e z}.
\]

\paragraph{Slow rate due to environment heterogeneity when $\lambda \ge \lambda_0 $.} 
Recall that $\lambda<\frac{1}{1+\gamma}$.
  Applying the triangle inequality yields that
  \begin{align*}
\abth{\alpha_E^z + \frac{1-\alpha_E^z}{1-\alpha_E}\kappa_E }
\ge \abth{\frac{\kappa_E}{1-\alpha_E}} - \pth{1+ \abth{\frac{\kappa_E}{1-\alpha_E}} }\alpha_E^z.
\end{align*}
For the first term, by the negativity and monotonicity in~\prettyref{lmm:kappa_E}, it suffices to show the lower bound at $\lambda = \lambda_0$.
Since $\lambda<\frac{1}{1+\gamma}$, then $\alpha_E = \frac{1}{2}\pth{(1-(1-\gamma)\lambda)^E + (1-(1+\gamma)\lambda)^E} $ decreases as $\lambda$ increases. 
{For $z\geq \exp\left\{-W_{-1}\left(-\frac{1-\gamma}{2(1+\gamma)}\right)\right\}, \text{ where } W_{-1} \text{ is the Lambert $W$ function} $, such that $(1+\gamma)\lambda_0\le \frac{1}{2E}$, we apply the lower bound in~\prettyref{lmm:kappa_E} and obtain that

\[\left|\frac{\kappa_E}{1-\alpha_E}\right| \geq \frac{\lambda_0 \gamma^2(E-1)}{4}\geq \frac{\frac{\log z}{(1-\gamma) z E} \gamma^2(E-1)}{4}\geq \frac{(E-1)}{4E} \gamma^2 \frac{\log z}{(1-\gamma)z}.\]}
Additionally, applying the upper bound in~\prettyref{lmm:kappa_E} yields 
\[
\pth{1+ \abth{\frac{\kappa_E}{1-\alpha_E}} }\alpha_E^z 
\le \frac{\nu_2^{zE}}{1-\gamma^2}=\frac{(1-(1-\gamma)\lambda)^{zE}}{1-\gamma^2}\le \frac{1}{(1-\gamma^2)z}.
\]
{Therefore,
\begin{align*}
    \left|{\alpha_E^z + \frac{1-\alpha_E^z}{1-\alpha_E}\kappa_E }\right|
&\ge \left|{\frac{\kappa_E}{1-\alpha_E}}\right| - \left({1+ \left|{\frac{\kappa_E}{1-\alpha_E}}\right| }\right)\alpha_E^z\\
&\geq \frac{(E-1)}{4E} \gamma^2 \frac{\log z}{(1-\gamma)z} -  \frac{1}{(1-\gamma^2)z} \\
&= \frac{1}{(1-\gamma^2)z} \left({(1+\gamma)\gamma^2\log(z)(E-1)/(4E) -1}\right)\\
& = \frac{1}{(1-\gamma)z} \left({\frac{\gamma^2\log(z)(E-1)-4E/(1+\gamma)}{4E}}\right). 
\end{align*}
When $r_1=r_2$, 
\begin{align*}
    \linf{\Delta_{zE}} = \frac{\abth{r_1+r_2}}{2(1-\gamma)} \abth{\beta_E^z} \\
    \geq \frac{\abth{r_1+r_2}}{2(1-\gamma)}\frac{E}{eT};
\end{align*}
    
otherwise,
\begin{align*}
\|{\Delta_{zE}}\|_{\infty}
&\ge \frac{\min\{\|(I-\bar P)Q^*\|_2, \|\bar P Q^*\|_2 \}}{\sqrt{2}} \cdot \max\left\{|\alpha_E^z + \frac{1-\alpha_E^z}{1-\alpha_E}\kappa_E|, \beta_E^z\right\}\\
&\ge \frac{1}{2} \min\left\{\abth{r_1-r_2}, \frac{1}{1-\gamma}\abth{r_1+r_2}\right\}\max\left\{|\alpha_E^z + \frac{1-\alpha_E^z}{1-\alpha_E}\kappa_E|, \beta_E^z\right\}\\
& \geq \frac{1}{2}\min\left\{\abth{r_1-r_2}, \frac{1}{1-\gamma}\abth{r_1+r_2}\right\}\max\left\{\frac{1}{(1-\gamma)z} \pth{\frac{\gamma^2\log(z)(E-1)-4E/(1+\gamma)}{4E}}, \frac{1}{e z}\right\}\\
& = \frac{1}{2}\min\left\{\abth{r_1-r_2}, \frac{1}{1-\gamma}\abth{r_1+r_2}\right\}\max\left\{\frac{E}{(1-\gamma)T} \pth{\frac{\gamma^2\log(z)(E-1)-4E/(1+\gamma)}{4E}}, \frac{E}{eT}\right\}. 
\end{align*}

We can choose $\log(z)\geq \frac{4E(\gamma+2)}{(1+\gamma)\gamma^2(E-1)}, E\geq 2$ so that $\left(\frac{\gamma^2\log(z)(E-1)-4E/(1+\gamma)}{4E}\right)\geq 1$. Then the first term inside the max operator is bigger. Then,
\[\|\Delta_{zE}\|_{\infty}  \geq \frac{1}{2}\min\left\{\abth{r_1-r_2}, \frac{1}{1-\gamma}\abth{r_1+r_2}\right\}\frac{E}{(1-\gamma)T}. \]



}

\newpage 

\section{Additional experiments}
\label{app: add: experiments}
\begin{figure}[h]
    \centering
    \begin{subfigure}[t]{0.475\textwidth}
    \includegraphics[width=1\textwidth]{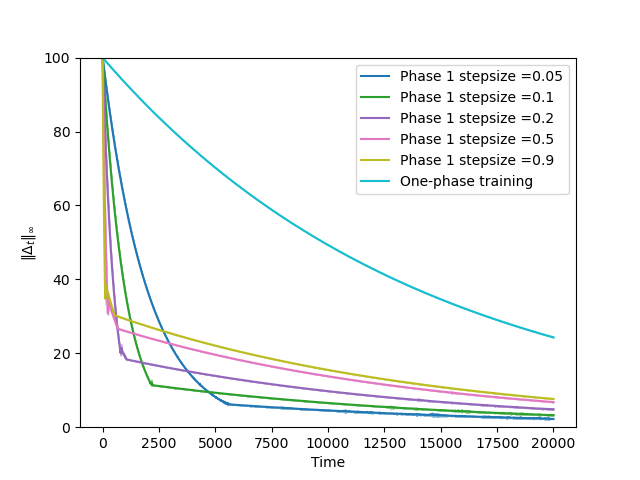}
    \caption{Phase 2 stepsize $\lambda=\frac{1}{\sqrt{T}}$} 
    \label{fig:two-phase a}
    \end{subfigure}
    \hfill
    \begin{subfigure}[t]{0.475\textwidth}
    \includegraphics[width=1\textwidth]{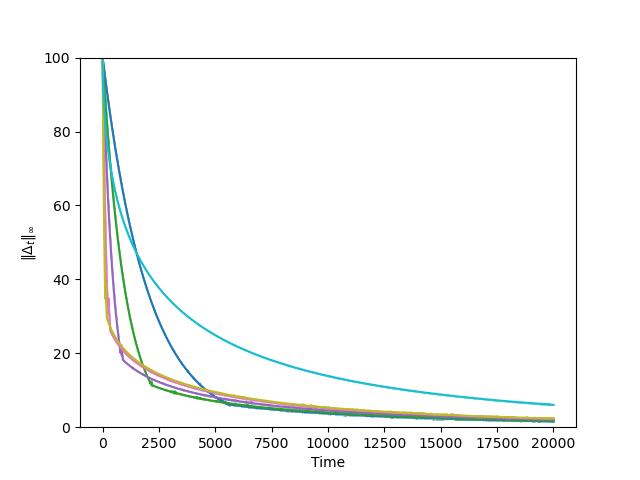}
    \caption{Phase 2 stepsize $\lambda_t=\frac{1}{\sqrt{t+1}}$} 
    \label{fig:two-phase b}
    \end{subfigure}
    \begin{subfigure}[t]{0.475\textwidth}
    \includegraphics[width=1\textwidth]{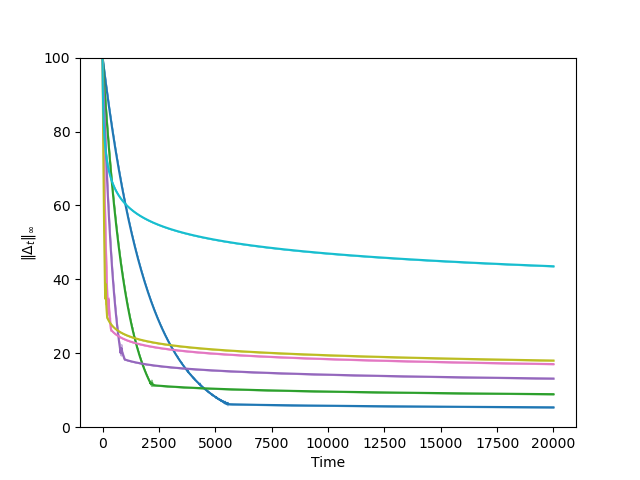}
    \caption{Phase 2 stepsize $\lambda_t=\frac{c+1}{t+c}$, where $c=10$} 
    \label{fig:two-phase c}
    \end{subfigure}
    \begin{subfigure}[t]{0.475\textwidth}
    \includegraphics[width=1\textwidth]{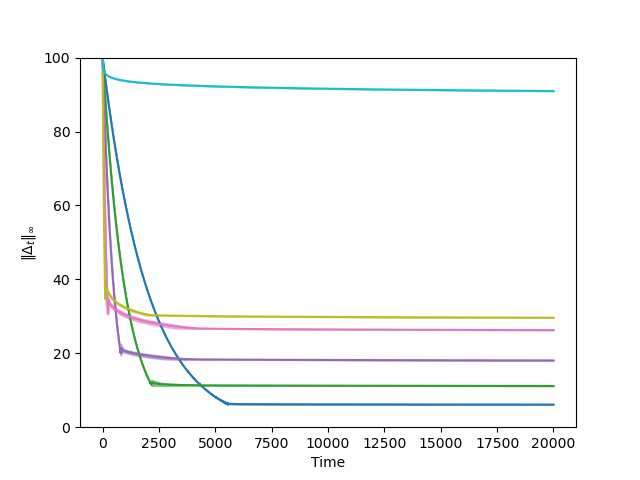}
    \caption{Phase 2 stepsize $\lambda_t=\frac{1}{t+1}$} 
    \label{fig:two-phase d}
    \end{subfigure}
    \begin{subfigure}[t]{0.475\textwidth}
    \includegraphics[width=1\textwidth]{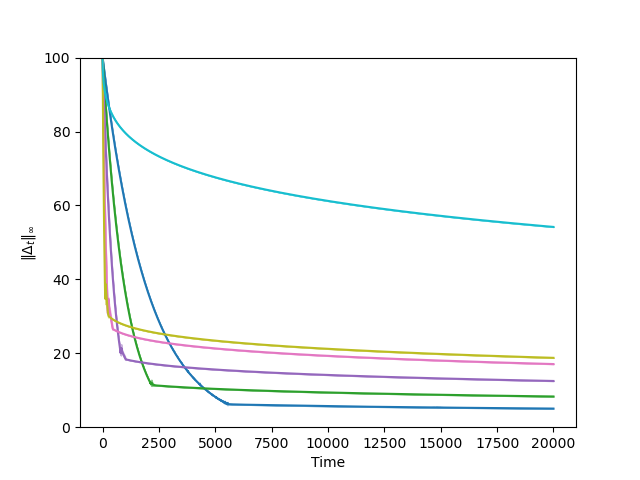}
    \caption{Phase 2 stepsize $\lambda_t=\frac{1}{(t+1)^{0.7}}$} 
    \label{fig:two-phase e}
    \end{subfigure}
    \caption{ Choosing different stepsizes for phases 1 and 2 leads to faster overall convergence. \(E\) = 10.}
    \label{fig:two-phase}
\end{figure}
\subsection{Impacts of $E$ on homogeneous settings.}
\label{app: exp: homo: E}
For the homogeneous settings, in addition to $E=10$, we also consider $E=\{1, 20, 40, \infty\}$, where $E=\infty$ means no communication among the agents throughout the entire learning process.  
Similar to Figure \ref{fig:two-phase-b}, there is no obvious two-phase phenomenon even in the extreme case when $E=\infty$. 
Also, though there is indeed performance degradation caused by larger $E$, the overall performance degradation is nearly negligible compared with the heterogeneous settings shown in Figures \ref{fig:two-phase-a} and \ref{fig:sync_int}. 
\begin{figure}[h]
\begin{subfigure}[t]{0.475\textwidth}
    \includegraphics[width=\textwidth]{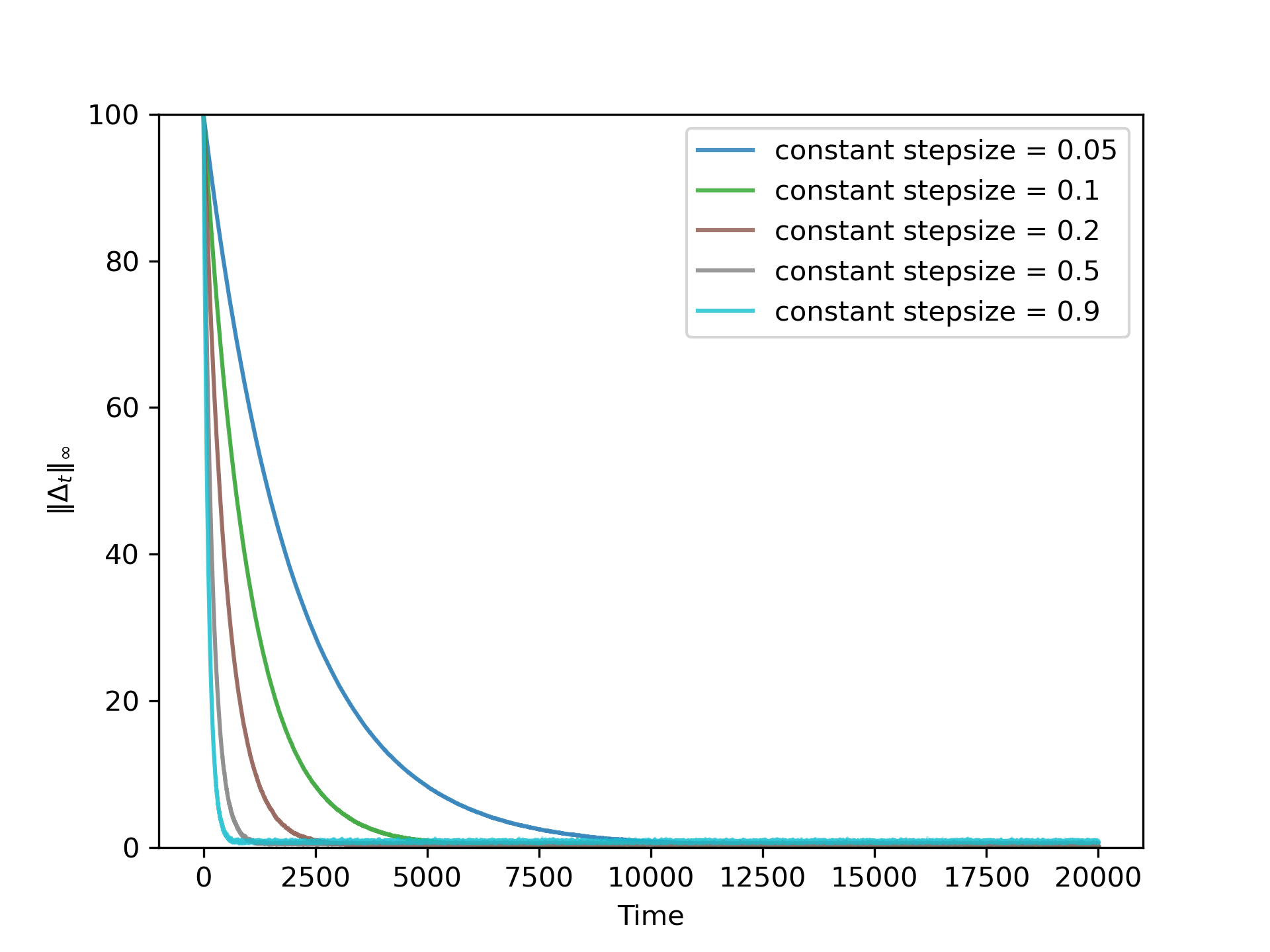}
    \caption{E=1}
    \label{fig-a}
  \end{subfigure}
  \hfill
  \begin{subfigure}[t]{0.475\textwidth}
    \includegraphics[width=\textwidth]{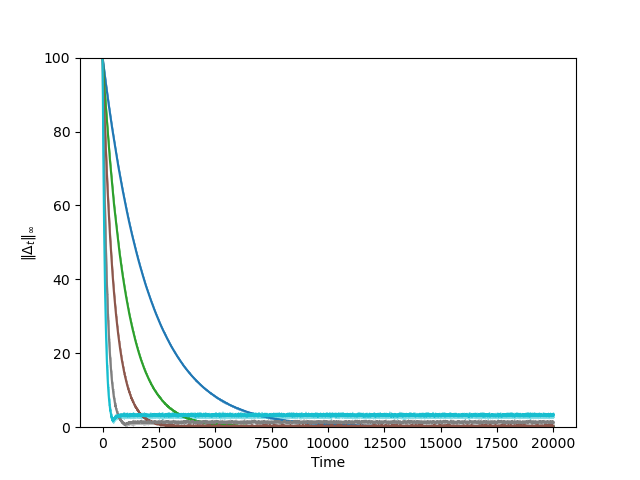}
    \caption{E=20}
    \label{fig-a}
  \end{subfigure}
  \hfill
  \begin{subfigure}[t]{0.475\textwidth}
    \includegraphics[width=\textwidth]{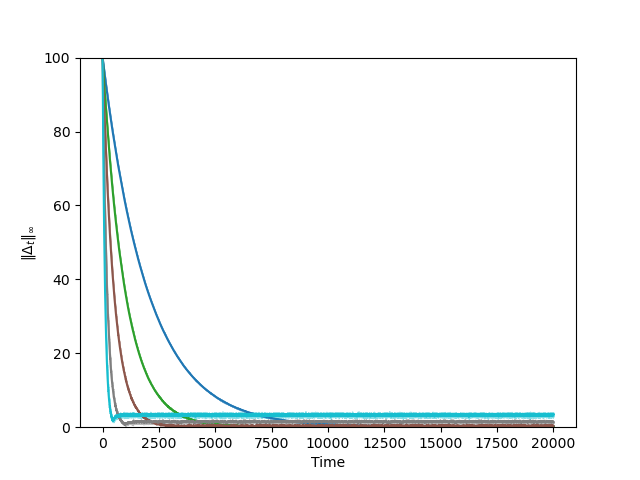}
    \caption{E=40}
    \label{fig-b}
  \end{subfigure}
    \hfill 
  \begin{subfigure}[t]{0.475\textwidth}
    \includegraphics[width=\textwidth]{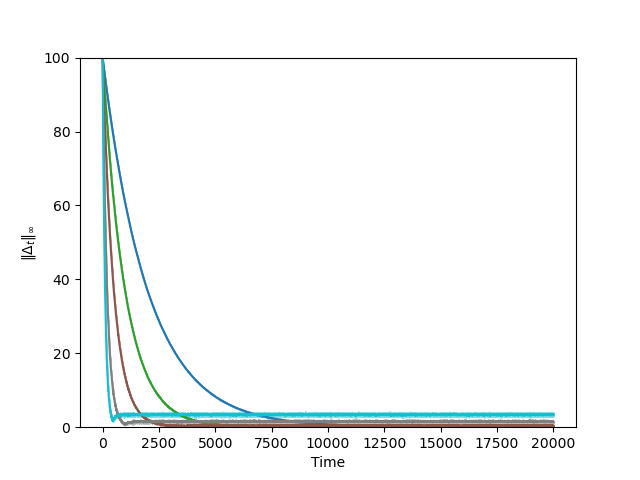}
    \caption{E=\(\infty\)}
    \label{fig-b}
  \end{subfigure}
  \caption{Homogeneous federated Q-learning with varying $E$.}
  \label{fig:homo_sync_int}
\end{figure}

\subsection{Impacts of $E$ on time-decaying stepsize}

\begin{figure}[H]
    \centering
    \includegraphics[width=0.6\linewidth]{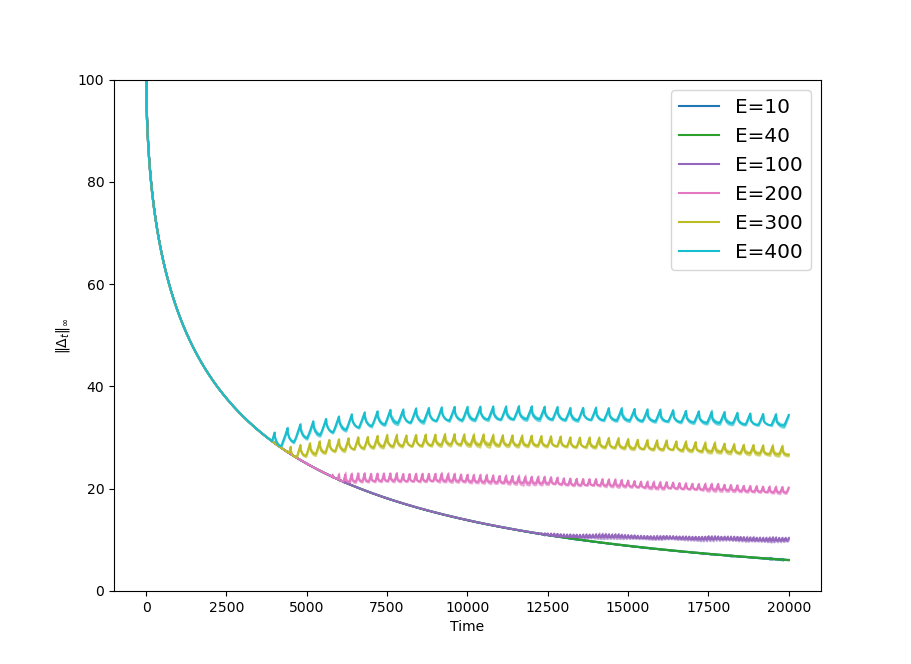}
    \caption{Using time-decaying stepsize $\lambda_t=\frac{1}{\sqrt{t+1}}$, the overall convergence becomes worse as $E$ increases}
    \label{fig:diffE}
\end{figure}

\begin{figure}[H]
    \centering
    \includegraphics[width=0.5\linewidth]{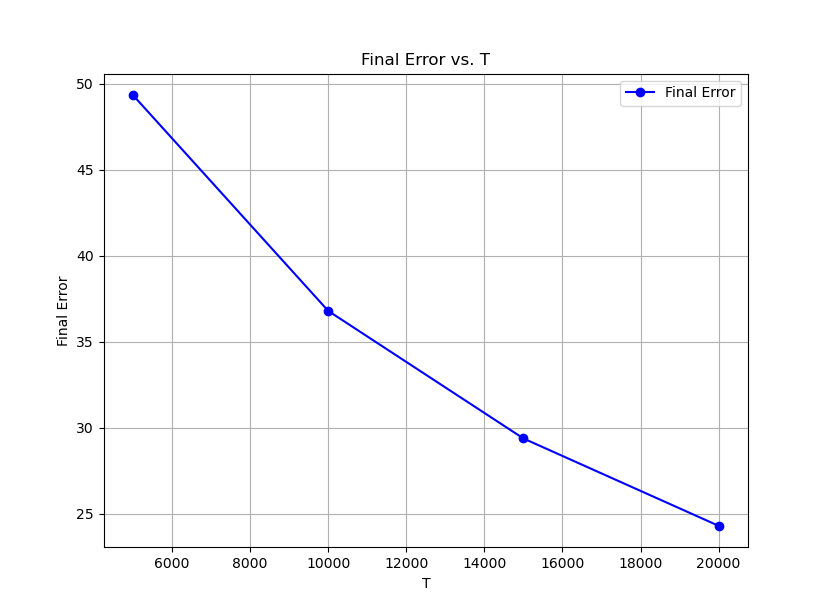}
    \caption{Final error versus T. It is clear that when choosing $\lambda=\frac{1}{\sqrt{T}}$, the final error decays as $T$ increases.}
    \label{fig:enter-label}
\end{figure}

\subsection{Different target error levels.}
\label{app: two-phase: stepsizes} 
\noindent In Figure \ref{fig:diff-tol}, we show the error levels that these training strategies can achieve within a time horizon $T=20,000$. 
The tolerance levels are $10\%, 5\%, 3\%,\text{ and } 1\%$ of the initial error $\linf{\Delta_0}$, respectively. 
At a high level, choosing different stepsizes for phases 1 and 2 can speed up convergence. 
 %
\begin{figure}
\begin{subfigure}[t]{0.475\textwidth}
  \vspace*{-0.3in}
    \includegraphics[width=\textwidth]{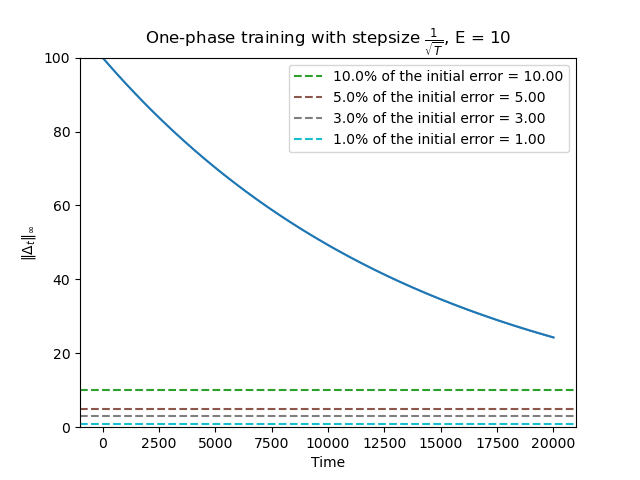}
    \caption{One common $\lambda = \frac{1}{\sqrt{T}}$ throughout. $\linf{\Delta_t}$ does meet any of the tolerance levels within 20000 iterations}
    \label{fig-a}
  \end{subfigure}
  \hfill
  \begin{subfigure}[t]{0.475\textwidth}
  \vspace*{-0.3in}
    \includegraphics[width=\textwidth]{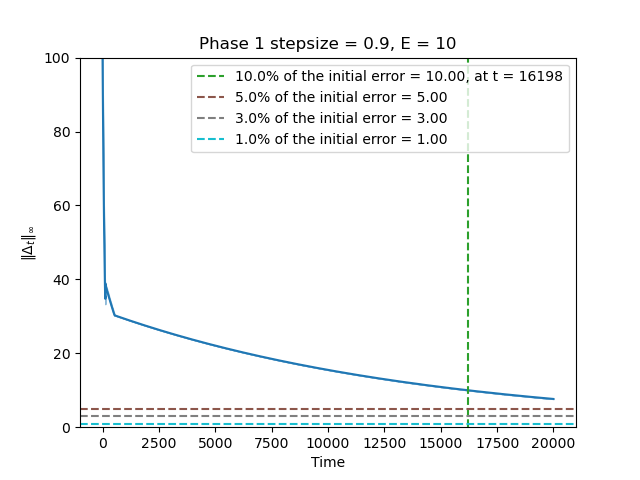}
    \caption{With a phase 1 stepsize of 0.9, it meets the 10\% tolerance level at iteration 16198.}
    \label{fig-b}
  \end{subfigure}

  \begin{subfigure}[t]{0.475\textwidth}
  \vspace*{-0.05in}
    \includegraphics[width=\textwidth]{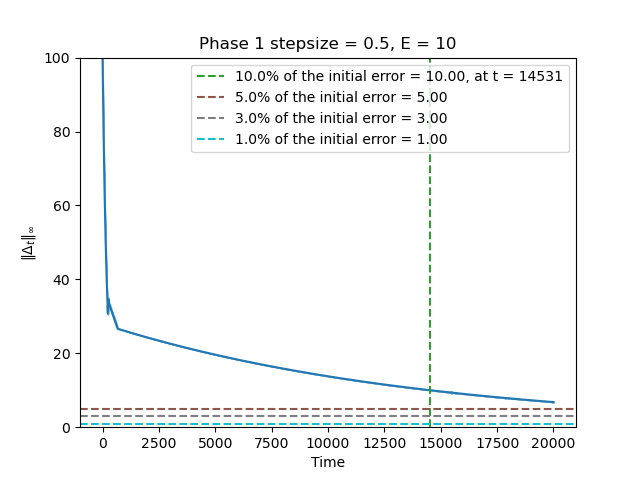}
    \caption{With a phase 1 stepsize of 0.5, it meets the 10\% tolerance level at iteration 14531.}
    \label{fig-c}
  \end{subfigure}
  \hfill
  \begin{subfigure}[t]{0.475\textwidth}
    \vspace*{-0.05in}
    \includegraphics[width=\textwidth]{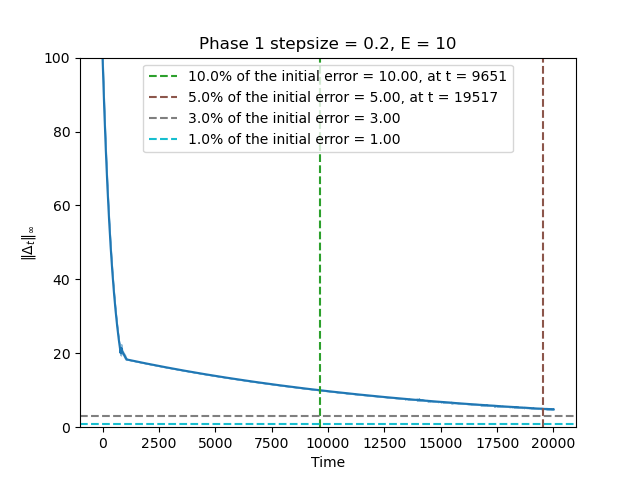}
    \caption{With a phase 1 stepsize of 0.2, it meets the 10\% and 5\% tolerance level at iterations 9651 and 19517, respectively.}
    \label{fig-d}
  \end{subfigure}

  \begin{subfigure}[t]{0.475\textwidth}
    \includegraphics[width=\textwidth]{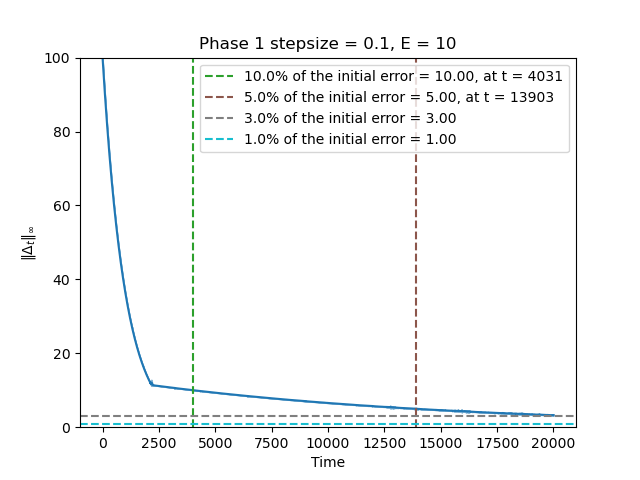}
    \caption{With a phase 1 stepsize of 0.1, it meets the 10\% and 5\% tolerance level at iterations 4031 and 13903, respectively.}
    \label{fig-e}
  \end{subfigure}
    \hfill
    \begin{subfigure}[t]{0.475\textwidth}
    \includegraphics[width=\textwidth]{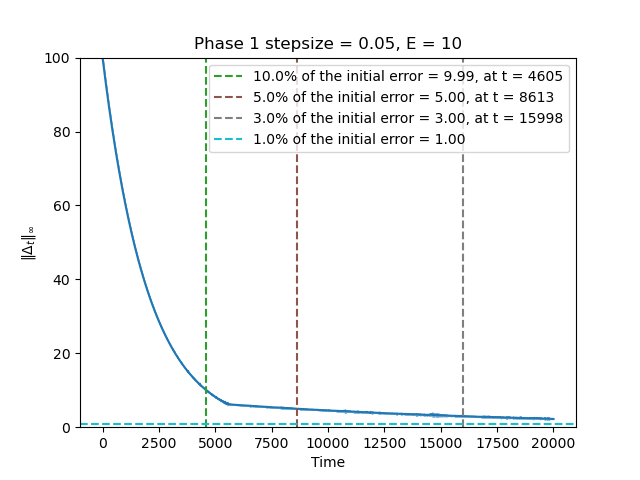}
    \caption{With a phase 1 stepsize of 0.05, it meets the 10\%, 5\%, and 3\% tolerance levels at iterations 4605, 8613, and 15998, respectively.}
    \label{fig-f}
  \end{subfigure}
  \caption{\footnotesize Convergence performance of different tolerance levels of different stepsize choices. The horizontal dashed lines represent the tolerance levels not met, while the vertical dashed lines indicate the iterations at which the training processes meet the corresponding tolerance levels.} 
  \label{fig:diff-tol}
\end{figure}

\end{document}